\documentclass{article}

\usepackage[english]{babel}

\usepackage[a4paper,top=2cm,bottom=2cm,left=3cm,right=3cm,marginparwidth=1.75cm]{geometry}

\usepackage{authblk}

\usepackage{amsmath,amssymb}
\usepackage{amsthm}
\usepackage{mathtools}

\usepackage{natbib}
\usepackage{graphicx}
\usepackage{bm}
\usepackage[colorlinks=true, allcolors=blue]{hyperref}
\usepackage{multirow, makecell}
\usepackage[table]{xcolor}
\usepackage{tcolorbox}
\usepackage{nicematrix}
\usepackage{caption}
\usepackage{subcaption}
\usepackage{nicematrix}
\usepackage{tikz}
\usepackage{xspace}

\DeclarePairedDelimiter\abs{\lvert}{\rvert}

\newcommand{\kernelf}{k}
\newcommand{\akernelf}{g}
\newcommand{\trsize}{n}
\newcommand{\drsize}{n_D}
\newcommand{\tasize}{n_T}
\newcommand{\algo}{\mathcal{A}}
\newcommand{\predictorvar}{F}

\newcommand{\drug}{d}
\newcommand{\targ}{t}
\newcommand{\Drug}{D}
\newcommand{\Targ}{T}
\newcommand{\predfun}{f}
\newcommand{\predfunDT}{\predfun_{\mathcal{\Drug},\mathcal{\Targ}}}
\newcommand{\predfunD}{\predfun_\mathcal{\Drug}}
\newcommand{\predfunT}{\predfun_\mathcal{\Targ}}

\newcommand{\Xkdkt}{\mathcal{X}^{\textnormal{\tiny IDIT}}}
\newcommand{\Xndkt}{\mathcal{X}^{\textnormal{\tiny ODIT}}}
\newcommand{\Xkdnt}{\mathcal{X}^{\textnormal{\tiny IDOT}}}
\newcommand{\Xndnt}{\mathcal{X}^{\textnormal{\tiny ODOT}}}

\newcommand{\testsetsize}{{\overline{\trsize}}}

\newcommand{\Cindex}{C-index\xspace}
\newcommand{\CDindex}{C$_{\Drug}$-index\xspace}
\newcommand{\CTindex}{C$_{\Targ}$-index\xspace}
\newcommand{\Aindex}{IC-index\xspace}

\newcommand{\Cutil}{C\xspace}
\newcommand{\CDutil}{C$_{\Drug}$\xspace}
\newcommand{\CTutil}{C$_{\Targ}$\xspace}
\newcommand{\Autil}{IC\xspace}

\makeatletter
\newcommand{\subalign}[1]{%
  \vcenter{%
    \Let@ \restore@math@cr \default@tag
    \baselineskip\fontdimen10 \scriptfont\tw@
    \advance\baselineskip\fontdimen12 \scriptfont\tw@
    \lineskip\thr@@\fontdimen8 \scriptfont\thr@@
    \lineskiplimit\lineskip
    \ialign{\hfil$\m@th\scriptstyle##$&$\m@th\scriptstyle{}##$\hfil\crcr
      #1\crcr
    }%
  }%
}
\makeatother

\title{Interaction Concordance Index: Performance Evaluation for Interaction Prediction Methods}
\author{Tapio Pahikkala}
\author{Riikka Numminen} \author{Parisa Movahedi}
\author{Napsu Karmitsa}
\author{Antti Airola}
\affil{Department of Computing, University of Turku, Turku, Finland}

\newtheorem{definition}{Definition}

\newtheorem{proposition}{Proposition}

\newtheorem{example}{Example}
\newtheorem{remark}{Remark}

\begin{document}
\maketitle

\begin{abstract}
Consider two sets of entities and their members' mutual affinity values, say drug-target affinities (DTA). Drugs and targets are said to interact in their effects on DTAs if drug's effect on it depends on the target. Presence of interaction implies that assigning a drug to a target and another drug to another target does not provide the same aggregate DTA as the reversed assignment would provide. Accordingly, correctly capturing interactions enables better decision-making, for example, in allocation of limited numbers of drug doses to their best matching targets. Learning to predict DTAs is popularly done from either solely from known DTAs or together with side information on the entities, such as chemical structures of drugs and targets. In this paper, we introduce interaction directions' prediction performance estimator we call interaction concordance index (IC-index), for both fixed predictors and machine learning algorithms aimed for inferring them. IC-index complements the popularly used DTA prediction performance estimators by evaluating the ratio of correctly predicted directions of interaction effects in data. First, we show the invariance of IC-index on predictors unable to capture interactions. Secondly, we show that learning algorithm's permutation equivariance regarding drug and target identities implies its inability to capture interactions when either drug, target or both are unseen during training. In practical applications, this equivariance is remedied via incorporation of appropriate side information on drugs and targets. We make a comprehensive empirical evaluation over several biomedical interaction data sets with various state-of-the-art machine learning algorithms. The experiments demonstrate how different types of affinity strength prediction methods perform in terms of IC-index complementing existing prediction performance estimators.
\end{abstract}

\section{Introduction}\label{sec:intro}

Predicting the value of a quantity associated with a pair of entities is an ubiquitous task in biomedical applications. Typical examples include predicting the existence or magnitude for drug-target \citep{pahikkala2015realistic}, drug-drug \citep{vilar2014similarity}, protein-protein \citep{ben2005kernel}, or protein-RNA \citep{bellucci2011predicting} affinities. 
These are often cast as supervised machine learning problems, where from a training data of, say, drug-target pairs with known affinity values, a learning algorithm infers a predictor for pairs with unknown affinity values. These learning algorithms are either based on off-the-shelf implementations of standard classification or regression methods (see e.g. \citet{yu2012systematic}) or on methods specifically tailored to the task. The latter include pairwise kernel methods  \citep{ben2005kernel,viljanen2021generalized}, specialized deep learning architectures \citep{Ozturk2018deepdta, Nguyen2020graphdta}, and matrix factorization methods \citep{zheng2013collaborative}.

In what follows, 
we use a specific example of predicting the affinity strengths between drugs and their possible targets.
This task involves predicting either a binary value, indicating whether the drug is a good match for the target or not, or a real-valued bioactivity measurement between the drug and the target. However, the concepts discussed are not limited to this particular case, but can be applied to a wide range of tasks that involve predicting outcomes for pairs of objects. Representative examples include queries and documents in information retrieval \citep{liu2009learning} as well as customers and products in the context of recommender systems \citep{herlocker2004evaluating}.

Let us denote a drug by $\drug \in \mathcal{\Drug}$ and a target by $\targ \in \mathcal{\Targ}$, where $\mathcal{\Drug}$ and $\mathcal{\Targ}$ denote the sets of \textbf{categorical} identities of drugs and targets, respectively. The observed \textbf{drug-target affinity} (DTA) value between $\drug$ and $\targ$ can be considered as a random variable $Y$ endowed with some unknown distribution.
To inspect DTA values' dependence on drugs and targets, they are popularly expressed as additive decompositions of four distinct components we call here as the \textbf{grand mean} (i.e. the average affinity), \textbf{drug main effect}, \textbf{target main effect} and \textbf{interaction effect} (see e.g. \citet{vanderweele2014tutorial,Bours2021tutorial} and references therein).
As the name suggests, the grand mean simply indicates the mean affinity value, regardless of the drug or target. The drug and target main effects can be roughly interpreted as the drugwise and targetwise average DTA differences from the grand mean. Finally, the interaction effect can be considered as the affinity values departure from the sum of the three other components. We refer to the sign of the interaction effect as the \textbf{direction of interaction}.

As a related but somewhat orthogonal work, we note that instead of considering interaction in the above described additive scale, it is also popularly considered in other scales, especially  in multiplicative and odds scales (see e.g. \citet{vanderweele2014tutorial,Bours2021tutorial,Spake2023itdepends} and references therein).  To switch scale from additive to one of the alternatives, one can simply use an appropriate link function to transform the affinity values (see e.g. \citet{Ronkko2022eight}). For example, one can switch to multiplicative scale by taking the logarithm of the affinity values and to odds scale by applying the logistic function on them, after which one can continue with the additive interaction analysis on the transformed design. Therefore, in this paper we focus only on the additive scale for  its simplicity and intuitive convenience. However, it is worth noting that the presence, or even the direction of interaction, are not necessarily preserved when the scale is switched, as is pointed out by several authors in the literature \citep{vanderweele2014tutorial,Bours2021tutorial,Ronkko2022eight,Spake2023itdepends}.

Interaction effect can be interpreted in multiple different ways. The most straightforward one is the drug main effects' dependence on the target in the sense that DTA values can not be considered simply as the sums of the drugs' and targets' main effects \citep{Spake2023itdepends}. In the opposite case, in which the interaction effect is absent, drug $\drug$ having stronger affinity with target $\targ$ than drug $\drug'$ would imply that $\drug$ also has stronger affinity with target $\targ'$ than $\drug'$. Another popular point of view is what is in the literature referred to as the ``public health argument'' (see e.g. \citep{vanderweele2014tutorial}) concerning the aggregate benefits (or costs) of assigning $\drug$ to $\targ$ and $\drug'$ to $\targ^*$ compared to the opposite assignment. For example, if there is a limited number of doses of a better but more expensive drug whereas a cheap but worse drug is in abundance, more patients can be cured by assigning the former drug for the targets on which the drugs' effect difference is larger.

We now shift our focus on predictors, usually regression functions of either DTA values or probabilities of the affinity's existence in binary classification.
A predictor can, analogously to $Y$, be expressed as a decomposition
\begin{equation} \label{eqn:decomposition}
\predfun(\drug,\targ)
=\predfun_C+\predfunD(\drug)+\predfunT(\targ)+\predfunDT(\drug,\targ),
\end{equation}
where $\predfun_C$, $\predfunD$, $\predfunT$ and $\predfunDT$ are terms depending on neither drug nor target, only drug, only target and both drug and target arguments, respectively. With the same terminology as above, we say that $\predfun_C$, $\predfunD$, $\predfunT$ and $\predfunDT$ model the grand mean, drug main, target main and interaction effects, respectively.

Next we briefly recap, how popularly used prediction performance estimators, like the classification accuracy, area under the receiver operating characteristic curve (AUC) (see e.g. \citet{fawcett2006introduction}) and concordance index (\Cindex) \citep{harrell1996multivariable} as well as their drugwise and targetwise variations behave with predictors only consisting of either constant, drug main, target main or both main terms. 
\begin{itemize}
\item \textbf{Constant functions}. With binary valued $Y$ and imbalance between the two classes, one can obtain classification accuracy close to the expected $Y$ value by simply always predicting it for all drug-target pairs. These observations have motivated the recommendation to use ranking-based performance estimators such as AUC for binary \citep{schrynemackers2013protocols} and \Cindex for ordinal classification or regression \citep{pahikkala2015realistic} when benchmarking biological affinity prediction methods. For these estimators any constant function gives trivial 0.5 level performance.
\item \textbf{Target symmetric} or \textbf{drug symmetric}: Let  $\predfun(\drug,\targ)=\predfun_C+\predfunD(\drug)$ be a predictor modeling only the drug main effect, that is, it may indicate that some drugs tend to have greater affinity values than others. However, for any drug, it predicts the same DTA value for all targets, and hence we call it target symmetric. Analogously, predictors that can be expressed as $\predfun(\drug,\targ)=\predfun_C+\predfunT(\targ)$ model only target main effects. Some studies have recommended macro averaged drug-wise or target-wise ranking-based prediction performance estimators (see e.g. \citet{pahikkala2013efficient,stock2014identification,ezzat2019computational,dewulf2021cold}). The drug-wise estimators give the trivial 0.5 level \Cindex or AUC performance for predictors only modeling the drug main effects.
The same occurs for functions modeling only the target main effects with target-wise estimators.
\item \textbf{Both main effects}.
Let $\predfun(\drug,\targ)=\predfun_C+\predfunD(\drug)+\predfunT(\targ)$ be \textbf{additively separable} regarding drugs and targets, indicating it is able to represent both drug and target main effects but not interaction effects. Typical examples of machine learning algorithms inferring additively separable predictors are the ones that train linear models. Additionally, generalized linear models, such as logistic regression, are additively separable on the scale corresponding to their link functions \citep{Ronkko2022eight,Spake2023itdepends}.
Recently, \citet{viljanen2021generalized} experimentally demonstrated that additively separable models can, in several biomedical interaction prediction benchmarks, achieve highly competitive performance in terms of the \Cindex, even without capturing the interaction effect. However, the absence of interaction term still implies a severe limitation. If $\predfun(\drug,\targ)>\predfun(\drug',\targ)$ then $\predfun(\drug,\targ' )>\predfun(\drug',\targ')$ for all drugs and targets. Thus, when ordering drugs based on the predictions, the obtained ranking is the same despite the target, and there exists a ``universal'' drug predicted to be the best match for all targets.
\end{itemize}
Motivated by the above observations, we introduce a statistic we call \textbf{interaction concordance index} (\Aindex), usable for estimating the prediction performance of both DTA predictors and learning algorithms used for inferring the predictors. For a DTA predictor, \Aindex estimates the probability of correctly predicting the interactions' directions, given that they exist. Accordingly, predictors that do not model the interaction term, will always have a trivial 0.5 prediction performance.
For learning algorithms, we propose four variations of \Aindex that we elaborate below.

Based on the above, the inability to predict interaction effects is a problem for the classical (generalized) linear models, but may be less of a concern for more expressive methods, such as deep neural networks or random forests. However, we show that for certain classes of learning algorithms, the ability also strongly depends on whether the DTA values are predicted for such drugs and targets that have some observed DTA values in the training data or for those that have not. To make this consideration exact, we adopt the concept known as the off-training-set (OTS) prediction performance \citep{wolpert1992connection}. It indicates how well a predictor performs on data, whose inputs are distinct from those of the training data it has been inferred from. This is in contrast to the more well-known concept of generalization performance that makes no such restriction, but is rather the expected performance over the same distribution from which the training data is drawn (see e.g. \citet{roos2005generalization}). In DTA prediction, in-training-set (ITS) data are the drug-target pairs for which there are observed DTA values in the training data, and all other drug-target pairs form the OTS data. Moreover, we say that \textbf{ITS drugs} are the ones associated with at least one observed DTA value with any target in the training data, and the rest are \textbf{OTS drugs}. The \textbf{ITS targets} and \textbf{OTS targets} are defined analogously. 
Accordingly, we consider the following subsequent partition of the OTS data \citep{pahikkala2015realistic}, that are illustrated in Figure~\ref{fig:sideinfo}. Namely, OTS drug-target pairs formed from:
\begin{itemize}
\item \textbf{in-training-set drugs and in-training-set targets} (IDIT),
\item \textbf{off-training-set drugs and in-training-set targets} (ODIT),
\item \textbf{in-training-set drugs and off-training-set targets} (IDOT),
\item \textbf{off-training-set drugs and off-training-set targets} (ODOT).
\end{itemize}
We propose four variations of IC-index for learning algorithms. Namely those that estimate their DTA prediction performance on IDIT, ODIT, IDOT and ODOT drug-target pairs.

As a related work, we also note that multiple studies have established evaluation protocols based on distinct cross-validation settings, depending on whether affinities are imputed between entities already present in the training set or whether generalization to pairs including at least one novel entity is required \citep{park2012flaws,schrynemackers2013protocols,heimonen2014properties,pahikkala2015realistic,xian2018zero,celebi2019evaluation,mathai2020validation, stock2020algebraic,dewulf2021cold}.

\begin{figure}[h!]
\centering
\begin{tcolorbox}
\includegraphics[width=1\linewidth]{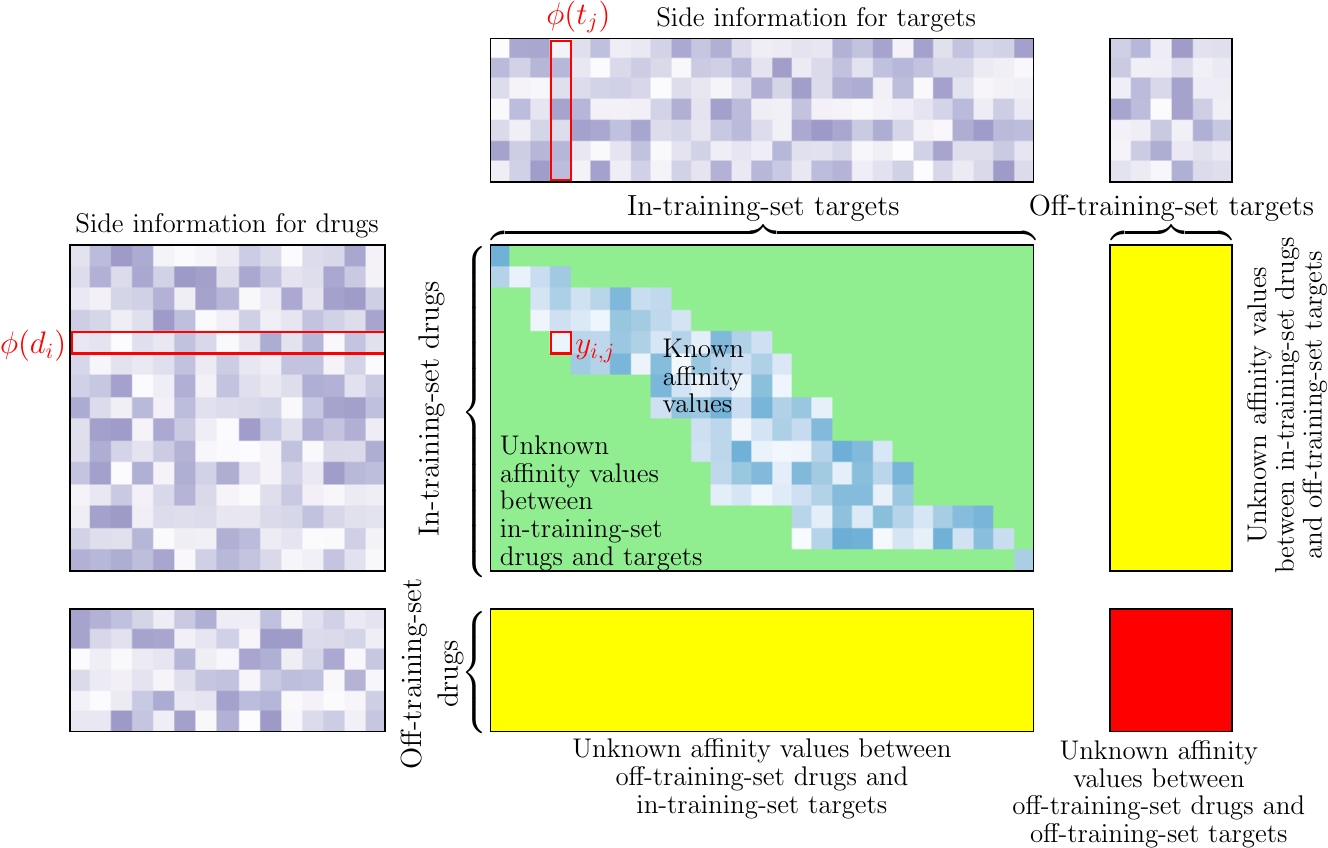}
\end{tcolorbox}
\caption{The figure illustrates the partition of drug-target pairs based on their components presence in the training data. The blue colored squares represent the in-training-set drug-target pairs having observed affinity values in the training data. The green area represent the off-training-set pair composed of in-training-set drugs and targets. The two yellow colored blocks represent pairs composed of either in-training-set drugs and off-training-set targets or vice versa. The red colored pairs are composed of off-training-set drugs and targets. The grey-colored blocks represent side information associated with either drugs or targets. For example, the row marked with $\phi(\drug_i)$ can be a feature representation associated to the drug $\drug_i$.
}
\label{fig:sideinfo}
\end{figure}

With the above described partition, we can analyze more in detail how the ability to predict interaction effects depends on the learning algorithms' access to what is commonly referred to as \textbf{side information} on drugs and targets a priori to training. From the theoretical perspective, side information or the lack of it can be considered to form a part of learning algorithms' inductive bias. We call a learning algorithm \textbf{permutation equivariant} with respect to drug identities, if it has no information on drugs other than their distinct categorical drug identities encoded into its inductive bias a priori to training. Permutation equivariance with respect to target identities is considered analogously.

A learning algorithm is said to be deterministic if it infers the same predictor every time it is rerun on the same training data. Such algorithm's permutation equivariance indicates that exchanging the identities of drugs $\drug$ and $\drug'$ in the training data results to the same identity exchange on the predictor learned from it. The identity exchange in the training data is interpreted as, for any target $\targ$, any DTA value observation for $(\drug,\targ)$ becoming that for $(\drug',\targ)$ and vice versa. For the predictor, the DTA value predicted for the pair $(\drug,\targ)$ becomes that for $(\drug',\targ)$ and vice versa. Then, the DTA predictions for $(\drug,\targ)$ and $(\drug',\targ)$ must be equal for OTS drugs, because the identity exchange has no effect in the training data, a property sometimes referred to as the principle of indifference. Conversely, any differences between OTS drugs' predicted DTA values implies that the learning algorithm is not permutation equivariant in this sense. Some information on these drugs differences beyond their distinct categorical identities must be encoded into the learning algorithm's inductive bias.

If a randomized learning algorithm (see e.g. \citet{elisseeff2005stabrandomized} and \citet{Oneto2020randomized}) is permutation equivariant, the drug or target identity exchanges are inherited by the distribution of predictors possibly learned from it. We show that the drug and target permutation equivariance of a randomized learning algorithm implies that its expected interaction directions' prediction performance is at the level 0.5 of random guessing except for the IDIT data. A predictor inferred by such a learning algorithm may be able to capture interactions also for the IDOT, ODIT and ODOT data by chance. Indeed, we demonstrate in our experiments that deceptively optimistic test results can sometimes emerge due to the randomness if test data is not sufficiently large.

As a representative example on how side information remedies the permutation equivariance, we consider matrix factorization based learning algorithms.
\begin{example}
Matrix factorization methods are one of the standard approaches for training recommender systems \citep{koren2009matrix}, but they are popularly applied in biomedical DTA prediction tasks as well. The methods are based on computing a low-rank factorization of a matrix whose rows and columns correspond to drugs and targets, respectively, and entries representing the DTA values, some of which are known and some unknown. DTA predictions for drug-target pairs $(\drug,\targ)$ are computed as $\predfun(\drug,\targ) = \langle\bm{v}_d,\bm{v}_t\rangle + b_d + b_t + b$, where $\bm{v}_d$ and $\bm{v}_t$ are embeddings of the drugs and the targets, $b_d$ and $b_t$ are drug and target dependent intercept terms, and $b$ is a global intercept \citep{koren2009matrix}. The first term allows capturing interaction effects in the IDIT area  (see Figure \ref{fig:sideinfo}). However, for the other areas, these values cannot be inferred from the known DTA values only. Some implementations may set $\bm{v}_d=\bm{0}$ and $b_d=0$, leading to a predictor $\predfun(\drug,\targ)= b_t + b$ only representing the target main effect on new drugs. Alternatively, implementations may instead resort to random values of $\bm{v}_d$ and $b_d$. Although the latter approach has a chance for random success, the expected interaction prediction performance remains at a random level. An analogous situation occurs when trying to generalize to novel targets. Furthermore, if both the drug and the target are outside the training set, only the global constant $b$ remains. However, if side information that differentiates the off-training-set drugs is available as per Figure \ref{fig:sideinfo}, then different embeddings and bias terms can be inferred for them.
\end{example}

The main contributions of the paper are as follows:
\begin{itemize}
\item We introduce interaction directions' prediction performance estimator for both fixed predictors and learning algorithms for inferring them that we call interaction concordance index (\Aindex).
\Aindex is invariant to both drug and target main effects, and hence only measures how well the interaction effects' directions are captured.
\item We review a representative suite of prediction performance estimators used for evaluating affinity prediction methods. Namely, classification accuracy and C-index as well as its drug- and targetwise (i.e. macro averaged) variations. Their invariance properties with respect to constant, drug symmetric and target symmetric as well as additively separable prediction functions are analyzed, and it is shown how the IC-index complements them.
\item We present rigorous definitions of four different subtypes of OTS affinity prediction learning problems, namely learning to predict affinities for IDIT, ODIT, IDOT and ODOT drug-target pairs. 
Learning algorithms' expected prediction performances on these problems admit representations as conditional expectations over the distribution of data, the condition corresponding to either IDIT, ODIT, IDOT or ODOT. We propose well-defined estimators for these quantities based on cross-validation, or more precisely repeated hold-out techniques.
\item 
We show that without having access to any side information on drugs or targets, no learning algorithm has better than random level expected IC-index on ODIT, IDOT or ODOT drug-target pairs. That is, better than random level expected IC-index is possible only on IDIT drug-target pairs. Consequently, all seemingly better than random IC-index estimates on other than IDIT pairs must be either due to random change or unobserved use of side information.
\item We make a comprehensive empirical evaluation over seven biomedical interaction data sets and 11 machine learning algorithms, demonstrating how different types of learning methods behave with respect to both the novel \Aindex and the other considered affinity prediction performance estimators.\footnote{Instructions for repeating the results are available at \url{github.com/TurkuML/IC-index-experiments}.}
\item We offer an open-source implementation of a binary search-tree-based algorithm for efficient computation of \Aindex.\footnote{Available at \url{github.com/TurkuML/Interaction-Concordance-Index}.}
\end{itemize}

\section{Preliminaries}\label{secPreli}
We start by introducing the notation and defining the most central concepts used in the paper in Section~\ref{secNotation}. Next, we define a number of utility functions that can be used to measure the performance of prediction functions or learning algorithms in Section \ref{secUtility}, and define the corresponding estimands and estimators in Sections~\ref{secEstimand}  and~\ref{secEstimator}.

\subsection{Notation}\label{secNotation}
We first give the main notations and definitions used in this paper. 
The notations used throughout the paper are also listed in Table~\ref{table:notations}.
\begin{table}[h!t]
{\small
\centering
 %\resizebox{0.99\textwidth}{!}{
\begin{tabular}{l l}
\hline %
\hline \\ [-0.9ex]
% acronyms
%DTI & drug-target interactions \\
$\mathcal{\Drug}$ & set of categorical drug identities\\
$\mathcal{\Targ}$ & set of categorical target identities\\
$\mathcal{X} = \mathcal{\Drug} \times \mathcal{\Targ}$ & set of inputs, i.e., set of drug-target pairs\\
$\mathcal{Z} = \mathcal{\Drug} \times \mathcal{\Targ} \times \mathbb{R}$ & set of drug-target pairs and their affinity values\\
$\drug \in \mathcal{D},\, \targ \in \mathcal{T}$ & drug identity, target identity \\
$x=(\drug,\targ) \in \mathcal{X}$ & input, drug-target pair \\
${y} \in \mathbb{R}$ & drug-target affinity value\\
$z=(\drug, \targ, y) \in \mathcal{Z}$ & datum, i.e. drug-target pair and its affinity value\\
$Z=(\Drug,\Targ,Y)$ & $\mathcal{Z}$-valued random variable\\
$\operatorname{P}_Z\in\mathcal{P}$ & probability distribution $Z$ is endowed with\\
$\mathcal{P}$ & collection of all probability distributions of data\\
$\bm{z}\in \mathcal{Z}^{\abs{\bm{z}}}$ & sequence of data of length $\abs{\bm{z}}$\\
$\operatorname{P}_{\bm{Z}}$ &degree $\abs{\bm{Z}}$ product probability distribution\\
$\mathcal{D}_{\bm{z}}$, $\mathcal{T}_{\bm{z}}$, $\mathcal{X}_{\bm{z}}$, $\mathcal{Z}_{\bm{z}}$ & subsets of drugs, targets, drug-target pairs, and data in $\bm{z}$\\
$\Xkdkt_{\bm{z}}, \ldots , \Xndnt_{\bm{z}}$& subsets of drug-target pairs not in $\mathcal{X}_{\bm{z}}$\\ 
%$\Ckdkt, \ldots , \Cndnt$& sets of constraints\\
$\mathbb{R}^\mathcal{X}$ & set of all predictor functions from $\mathcal{X}$ to $\mathbb{R}$\\
$\predfun:\mathcal{X}\rightarrow \mathbb{R}$ & predictor function\\
$\algo$& learning algorithm\\
$\predfun_{\bm{z}}\in\mathbb{R}^\mathcal{X}$ & predictor learned from $\bm{z}$ by $\algo$ \\
$\predictorvar_{\bm{z}}$ & $\mathbb{R}^\mathcal{X}$-valued random element learned from $\bm{z}$ by $\algo$ \\
$\kernelf:\mathcal{R}\rightarrow\left\{0,1/2,1\right\}$%,\, \kernelf_\predfun$
 & (inner) utility function for predictor $\predfun$ \\
$\mathcal{R}\subseteq\mathcal{Z}^{\abs{\kernelf}}$ & domain of $\kernelf$ \\
$\abs{\kernelf}$ & degree of $\kernelf$\\
$\akernelf:\mathcal{C}\rightarrow\left\{0,1/2,1\right\}$%,\, \akernelf_\algo$
 & (outer) utility function for learning algorithm $\algo$ \\
$\mathcal{C}\subseteq\mathcal{Z}^{\abs{\kernelf}}$ & domain of $\akernelf$\\
$\abs{\akernelf}$ & degree of $\akernelf$\\
$\theta$ &distribution level utility, the estimand\\
$\widehat{\theta}$& estimator of $\theta$\\
$\bm{s} \in \mathcal{Z}^{\abs{\bm{s}}}$ & sample of data of length $\abs{\bm{s}}$\\
$\sigma$ & injection from $\{1,\ldots,\abs{\kernelf}\}$ to $\{1,\ldots,\abs{\bm{s}}\}$\\
$\sigma\cdot\bm{s}$& sequence of entries of $\bm{s}$ from $\abs{\kernelf}$ distinct positions determined by $\sigma$\\
$H(a)$& Heaviside function with parameter $a$\\
$\operatorname{E}$ & expectation operator\\
$\trsize$ & size of the training data \\
$\drsize, \, \tasize$  & numbers of unique drugs and targets \\
$G%(\cdot,\cdot)
:\operatorname{P}_{Z},\algo\mapsto\theta$&collection of learning problems associated with utility $\akernelf$\\
$\operatorname{\Pi}_{\mathcal{D}}$, $\operatorname{\Pi}_{\mathcal{T}}$ & finitary symmetric groups on drug and target identities\\
\hline
%\hline %
\hline \\ [-0.9ex]
\end{tabular}}
\caption{Notations.} 
%\caption{Acronyms and notations.}
\label{table:notations}
%}
\end{table}
In what follows, $X = (\Drug, \Targ)$ denotes the bi-variate random variable for a single drug-target pair. Here, $\Drug$ and $\Targ$ are \textbf{categorical} (i.e., their values are from a set without any known order or structure) random variables with values in the sets $\mathcal{D}$ and $\mathcal{T}$, respectively.
We denote the realizations of these random variables as $x = (\drug, \targ)$. The \textbf{drug-target affinity} (DTA) value is denoted with a random variable $Y$ with realization $y$, whose values are real or ordinal. In addition, we denote the tri-variate random variable consisting of a drug-target pair and its DTA value as $Z=(\Drug,\Targ, Y)$ and its realizations as $z=(\drug, \targ, y)$. The variable $Z$ is endowed with an unknown joint probability distribution $\operatorname{P}_Z\in\mathcal{P}$, where $\mathcal{P}$ denotes the collection of all probability distributions of data. We also use the symbols $\mathcal{X}=\mathcal{D}\times\mathcal{T}$ and $\mathcal{Z}=\mathcal{D}\times\mathcal{T}\times\mathbb{R}$ to represent the sets of values for $X$ and $Z$, respectively. 
For sequences of data or random variables, we use bold notation. For example, $\bm{z}\in\mathcal{Z}^{\abs{\bm{z}}}$ denotes a sequence of data of length $\abs{\bm{z}}$ and $\operatorname{P}_{\bm{Z}}$ denotes the probability distribution for the sequences of data drawn independently from $\operatorname{P}_{Z}$.
\begin{remark}The independence assumption rarely holds in practical applications, and with DTA prediction learning problems, it tends to be violated even more often. Namely, each drug $\drug$ in a sample tends to be encountered as part of several data, for example, as a part of pair $x = (\drug,\targ)$ and $x'=(\drug,\targ')$, and the same applies to targets. This phenomenon can be seen in Figure~\ref{fig:sideinfo}, in which the sample contains only a single datum whose drug and target are not parts of any other datum in the sample.
\end{remark}

Let
\begin{align*}
&\predfun: \,\mathcal{X} \rightarrow \mathbb{R}
\end{align*}
be a function that maps a drug-target pair to a real-valued prediction of their DTA. In what follows, we refer to $\predfun$ as a \textbf{predictor}.

Further, for any sequence $\bm{z}$ of data, we use the following notation for subsets of drugs, targets, drug-target pairs and data:
\begin{alignat*}{5}
	&\mathcal{D}_{\bm{z}}
	&&=\{\drug\mid z_i=(\drug,\targ,y)\textnormal{ for some }1\leq i\leq\arrowvert\bm{z}\arrowvert,\targ\in\mathcal{T}\textnormal{ and }y\in\mathbb{R}\}\\
	&\mathcal{T}_{\bm{z}}&&=\{\targ\mid z_i=(\drug,\targ,y)\textnormal{ for some }1\leq i\leq\arrowvert\bm{z}\arrowvert,\drug\in\mathcal{D}\textnormal{ and }y\in\mathbb{R}\}\\
&\mathcal{X}_{\bm{z}}&&=\{(\drug,\targ)\mid z_i=(\drug,\targ,y)\textnormal{ for some }1\leq i\leq\arrowvert\bm{z}\arrowvert\textnormal{ and }y\in\mathbb{R}\}
\\
&\mathcal{Z}_{\bm{z}}&&=\{(\drug,\targ,y)\mid z_i=(\drug,\targ,y)\textnormal{ for some }1\leq i\leq\arrowvert\bm{z}\arrowvert\}\;.
\end{alignat*}
That is, $\mathcal{D}_{\bm{z}}\subseteq\mathcal{D}$, $\mathcal{T}_{\bm{z}}\subseteq\mathcal{T}$, $\mathcal{X}_{\bm{z}}\subseteq\mathcal{X}$, and $\mathcal{Z}_{\bm{z}}\subseteq\mathcal{Z}$ denote, respectively, the sets of drugs, targets, drug-target pairs, and data that occur at least once in the sequence $\bm{z}$.

When $\bm{z}$ is training data for a learning algorithm, we denote with $\mathcal{X}_{\bm{z}}$ the set of ITS drug-target pairs and with its complement $\mathcal{X}\setminus \mathcal{X}_{\bm{z}}$ the set of OTS drug-target pairs. The latter is further divided along the following partition:
\begin{definition}[Partition of off-training-set drug-target pairs]\label{def:otspartition}
Let us denote the four disjoint subsets of the OTS drug-target pairs $\mathcal{X}\setminus \mathcal{X}_{\bm{z}}$ as
\begin{align*}
\begin{aligned}
&\Xkdkt_{\bm{z}} &&=(\mathcal{D}_{\bm{z}}\times\mathcal{T}_{\bm{z}})\setminus\mathcal{X}_{\bm{z}}  &&\phantom{WW}\Xkdnt_{\bm{z}} &&=\mathcal{D}_{\bm{z}}\times(\mathcal{T}\setminus\mathcal{T}_{\bm{z}})\\[2ex]
&\Xndkt_{\bm{z}}&&=(\mathcal{D}\setminus\mathcal{D}_{\bm{z}})\times\mathcal{T}_{\bm{z}} &&\phantom{WW}\Xndnt_{\bm{z}}&&=(\mathcal{D}\setminus\mathcal{D}_{\bm{z}})\times(\mathcal{T}\setminus\mathcal{T}_{\bm{z}})
\end{aligned}
\end{align*}
that we refer to IDIT, IDOT
ODIT and ODOT drug-target pairs, respectively.
\end{definition}
\noindent These subsets are illustrated in Figure~\ref{fig:sideinfo}. They also coincide with the four experimental settings considered by \citet{pahikkala2015realistic}. 

Next we revise the definitions for deterministic and randomized learning algorithms and define a general form of a utility function that is used as basis for the specialized ones below. For learning algorithms, we use the following definition  (see e.g. \citet{elisseeff2005stabrandomized,Oneto2020randomized} and references therein for more in depth treatment of randomized learning algorithms):
\begin{definition}[Learning algorithm]\label{algodef}
A \textbf{deterministic} learning algorithm is considered as the mapping
\begin{alignat*}{3}
\begin{array}{lllll}
\algo:&\bigcup_{\arrowvert\bm{z}\arrowvert\in\mathbb{N}}\mathcal{Z}^{\arrowvert\bm{z}\arrowvert}&\rightarrow&\mathbb{R}^\mathcal{X}\smallskip\\ 
&\algo\left(
\bm{z}\right)&\mapsto&\predfun_{\bm{z}}
\end{array}\;,
\end{alignat*}
where $\mathbb{R}^\mathcal{X}$ denotes the set of all possible functions from $\mathcal{X}$ to $\mathbb{R}$.
With the subscripts in $\predfun_{\bm{z}}$, we stress that the predictor is inferred from the sequence $\bm{z}$ of training data. A \textbf{randomized} learning algorithm may infer different functions from the same training data if the learning process is repeated. To encapsulate this randomness, we use a $\mathbb{R}^\mathcal{X}$-valued random element $\predictorvar_{\bm{z}}$ in place of $\predfun_{\bm{z}}$, endowed with a probability distribution
\begin{align*}
\operatorname{P}[\predictorvar_{\bm{z}}\in\mathcal{F}]
\end{align*}
where $\mathcal{F}\subseteq\mathbb{R}^{\mathcal{X}}$ is any measurable subset of predictors.
Deterministic algorithms can obviously be considered as special cases of the randomized ones with the distribution's support only consisting of a single predictor.
\end{definition}

\begin{definition}[Utility function]\label{genericutility}
Utility functions are denoted by symbols $\kernelf$ and $\akernelf$ to indicate the utility value of a fixed predictor and that of a learning algorithm, in the respective order. In addition, their domains are denoted by symbols $\mathcal{R}$ and $\mathcal{C}$, respectively.
We may use the curry notation $\kernelf_\predfun$ to stress that the utility of a fixed predictor $\predfun$ under consideration is evaluated on the data arguments.
Similarly, the notation $\akernelf_\algo$ can be used to stress that the utility of a fixed learning algorithm $\algo$ is evaluated. However, both $\predfun$ and $\algo$ can also be considered as additional free arguments for $\kernelf$ and $\akernelf$, respectively, as we do in some of the forthcoming definitions. In this paper, all utilities are functions of the form
\begin{align*}
\kernelf:\mathcal{R}&\rightarrow
\left\{0,\frac{1}{2},1\right\}
\textnormal{ with }\mathcal{R}\subseteq\mathcal{Z}^{\abs{\kernelf}}\;,
\end{align*}
whose domain $\mathcal{R}$ consists of a subset of data sequences of length $\abs{\kernelf}$. In what follows, we refer to $\abs{\kernelf}$ as the \textbf{degree} and $\mathcal{R}$ as the \textbf{restriction} of the utility.
%The utility values 1, 0.5 and 0 can be interpreted simply as answers of type ``yes'', ``borderline'' and ``no''.
%If the utility of $\predfun$ is $\kernelf_\predfun(\bm{z})$, the utility of $-\predfun$ is $1-\kernelf_\predfun(\bm{z})$ for all $\bm{z}\in\mathcal{R}$. 
\end{definition}

\begin{remark}[Default utility value 0.5]
All utilities considered in this paper have their default value 0.5, interpreted in the following sense. Their domains have ``as many'' members for which the utility value is 0 as those for which it is 1 (omitting the unnecessarily cumbersome technical details concerning infinite and uncountable domains). Consequently, for any value $\theta\in[0, 1]$, there are ``as many'' probability distributions $\operatorname{P}_Z(z)\in\mathcal{P}$ for which the expected utility value is $\theta$ as those for which it is $1-\theta$.
\end{remark}

Lastly, the well-known Heaviside function
\begin{align}\label{heavisidefun}
H\left(a\right) = \begin{cases} 0 \quad\text{if } a < 0\\ \frac{1}{2} \quad\text{if } a = 0 \\ 1 \quad\,\text{if } a > 0\end{cases}
\end{align}
is used as the main component of defining both the classical utility function formalizations---namely those for binary classification and rank concordance---as well as for interaction concordance.

\subsection{Utility Functions}\label{secUtility}
We now define some utility functions relevant for our analysis. The first is what we call \textbf{binary classification utility}:
\begin{definition}[Binary classification utility]\label{def:accuracy}
Let $z=(x,y)$ and $\predfun$ be a predictor. The binary classification utility function
\[
\kernelf^{\textnormal{Acc}}_\predfun(z)=H(y\predfun(x)).
\]
indicates whether the sign of $y$ is correctly predicted by the value of $\predfun$ on $x$.
\end{definition}\smallskip
\noindent This utility can be used with binary valued, say $-1$ or $+1$, or continuous DTA values.
In the latter case, only the DTA values' sign is accounted for. %Zero is considered as a threshold value for continuous values or non-response prediction for binary labels.

The second utility, what we call here rank concordance or simply as \textbf{concordance}, determines whether the order of the predicted DTA values of two data points matches that of the observed ones. 
Suppose we have affinity observations for two drug-target pairs, say $z=(\drug,\targ,y)$ and $z'^*=(\drug',\targ^*,y'^*)$, such that $y>y'^*$. Concordance indicates whether the predicted DTA value $\predfun(x)$ for the first drug-target pair $x=(\drug,\targ)$ is greater than the predicted  DTA value $\predfun(x'^*)$ for the second pair $x'^*=(\drug',\targ^*)$.
\begin{definition}[Concordance]\label{concordanceindicatordef}
Concordance is expressed as the utility function
\begin{align}\label{concordanceutility}
\kernelf^{\textnormal{\Cutil}}_\predfun(z,z'^*)=H\left(f(x)-f(x'^*)\right),\;
\end{align}
restricted to the set
\begin{align}\label{concordancerestriction}
\mathcal{R}^{\textnormal{\Cutil}}=\{(z,z'^*)\in\mathcal{Z}^2\mid y>y'^*\}
\end{align}
consisting of all possible pairs of data, such that the  DTA value of the former datum is greater than that of the latter.
\end{definition}
\noindent In the literature (see e.g. \citet{Newson2002parameters}), concordance is often defined without the restriction \eqref{concordancerestriction} and scaled between -1 and 1. Here, we resort to the above definition for conformity with the other considered utilities.

Finally, we define what we call \textbf{drugwise} and \textbf{targetwise concordances}:
\begin{definition}[Drugwise and targetwise concordance]\label{drugntargetwisecidef}
For drugwise and targetwise concordances, the domain of the concordance utility $\kernelf_\predfun$ as per Definition \ref{concordanceutility} is further restricted to two data associated with the same drug for drugwise concordance
\[
\mathcal{R}^{\textnormal{\CDutil}}=\{(z,z')\in\mathcal{Z}^2\mid y>y', \drug=\drug'\},
\]
and with the same target for targetwise concordance
\[
\mathcal{R}^{\textnormal{\CTutil}}=\{(z,z^*)\in\mathcal{Z}^2\mid y>y^*, \targ=\targ^*\}.
\]
\end{definition}

\subsection{Distribution Level Utility}\label{secEstimand}

Based on the above-presented utilities on data, we now define the corresponding quantities on the underlying distributions of data. We refer to this type of quantity as the \textbf{distribution level utility} or simply as the \textbf{estimand}, since it can be considered as the aim of prediction performance estimation. As a simple rule of thumb, these quantities can be expressed as conditional expectations of utility functions, whose conditions correspond to the restrictions on the utility functions' domains.
\begin{definition}[Estimand]
\label{genericestimand}
Let $\kernelf$ be a utility function of degree $\arrowvert\kernelf\arrowvert$ and let $\mathcal{R}\subseteq\mathcal{Z}^{\arrowvert\kernelf\arrowvert}$ be its domain. Moreover, let $\mathcal{P}$ be a collection of probability distributions on data such that $\operatorname{P}_{\bm{Z}}[\mathcal{R}]>0$ for all $\operatorname{P}_Z\in\mathcal{P}$, where $\operatorname{P}_{\bm{Z}}$ denotes the degree $\arrowvert\kernelf\arrowvert$ product probability distribution. Then, the distribution level utility is
\begin{align*}
\theta
&=\operatorname{E}
[\kernelf(\bm{Z})\mid\bm{Z}\in\mathcal{R}]\;,
\end{align*}
where the expectation is taken over $\operatorname{P}_Z$.
\end{definition}
The next example presents what we call the \textbf{distribution level concordance}. In the literature, it is sometimes referred to as the distribution level Somer's D correlation if scaled between -1 and 1 (see e.g. \citep{Newson2002parameters}).
If the affinity values are binary, $\theta$ is sometimes called the distribution AUC, population AUC \citep{fawcett2006introduction}% the probability of superiority \citep{DeNeve2017MWinteraction},
or the Mann-Whitney parameter of the distribution \citep{fay2018confidence}, among many other names.
\begin{example}
[Distribution concordance]
%[Concordance probability]
\label{concordanceexample2}
For the concordance utility as per Definition~\ref{concordanceindicatordef}, the estimand becomes
\begin{align*}
\theta^{\textnormal{\Cutil}}_\predfun&=\operatorname{E}_{Z,Z'^*}\left[H\left(\predfun(X)-\predfun(X'^*)\right)\mid Y>Y'^*\right]\;,
\end{align*}
that can be referred to as the distribution level concordance.
\end{example}
For the binary classification utility and drugwise and targetwise concordances, similar procedures apply.

\subsection{Utility Estimation from a Sample of Data}\label{secEstimator}
Now, assume that we have access to a sequence $\bm{s}\in \mathcal{Z}^{\abs{\bm{s}}}$ of $\abs{\bm{s}}$ data drawn independently from $\operatorname{P}_Z$.
The following definition presents what we call simply as \textbf{estimator} for any of the estimands $\theta$ as per Definition~\ref{genericestimand}. We note that while one can consider any real-valued function of the sample as an estimator of $\theta$, in this paper we only focus on this type.
\begin{definition}[Estimator]\label{genericestimatordef}
Let $\kernelf$ be a utility function of degree $\arrowvert\kernelf\arrowvert$ and $\mathcal{R}\subseteq\mathcal{Z}^{\arrowvert\kernelf\arrowvert}$ its restriction. In addition, let $\bm{s}\in \mathcal{Z}^{\abs{\bm{s}}}$ be an IID sample of data drawn from the unknown distribution $\operatorname{P}_Z$.
With $\sigma=(i_1,\ldots,i_{\abs{\kernelf}})$ we denote any sequence of $\abs{\kernelf}$ distinct integers selected from $\{1,\ldots,\abs{\bm{s}}\}$ without repetition. In other words, the index $i_j$ is the image of $j$ on an arbitrary injection from $\{1,\ldots,\abs{\kernelf}\}$ to $\{1,\ldots,\abs{\bm{s}}\}$. Moreover, with $\sigma\cdot\bm{s}\in \mathcal{Z}^{\abs{\kernelf}}$, we denote a sequence of data consisting of the entries of $\bm{s}$ from the $\abs{\kernelf}$ distinct positions determined by $\sigma$. 
We consider estimators, whose values on a sample of data are
\begin{align}\label{genericestimator}
\widehat{\theta}(\bm{s})
&=\abs{\mathcal{I}}^{-1}
\sum_{\sigma\in\mathcal{I}}\kernelf(\sigma\cdot\bm{s})
\end{align}
if $\abs{\mathcal{I}}>0$ and $\widehat{\theta}(\bm{s})=0.5$ otherwise, where
\[
\mathcal{I}=\{\sigma\mid\sigma\cdot\bm{s}\in\mathcal{R}\}
\]
denotes the set of all index sequences, such that the corresponding data sequences $\sigma\cdot\bm{s}$ are conformable with the restriction $\mathcal{R}$.
\end{definition}
To concretize this abstract and generic definition with a practical example, we present the well-known \textbf{concordance index} (\Cindex), an estimator of the distribution level concordance as per Example~\ref{concordanceexample2}. \Cindex is also known as the Somers' D statistic when scaled between -1 and 1, (see e.g. \citet{Newson2002parameters}). If the DTA values are binary, C-index reduces to AUC, also known as the Mann-Whitney statistic (see e.g. \citet{fawcett2006introduction}).
\begin{example}[\Cindex and AUC]\label{cindexandaucexample}
An estimator $\widehat{\theta}^{\textnormal{\Cutil}}_{\predfun}(\bm{s})$ of the estimand $\theta^{\textnormal{\Cutil}}_\predfun$ as per Example~\ref{concordanceexample2} is obtained by substituting $\kernelf^{\textnormal{\Cutil}}_\predfun$ from (\ref{concordanceutility}) and $\mathcal{R}^{\textnormal{\Cutil}}$ from (\ref{concordancerestriction}) into (\ref{genericestimator}):
\[
\widehat{\theta}^{\textnormal{\Cindex}}_{\predfun}(\bm{s})
=\abs{\mathcal{I}}^{-1}\sum_{(i,j)\in\mathcal{I}}H(\predfun(x_i)-\predfun(x_j))\;,
\]
where
\[
\mathcal{I}=\{(i,j)\mid y_i>y_j\}\;.
\] 
\end{example}

If we adopt either the drug-wise $\mathcal{R}^{\textnormal{\CDutil}}$ or target-wise $\mathcal{R}^{\textnormal{\CTutil}}$ restrictions for concordance as per Definition~\ref{drugntargetwisecidef}, we obtain the macro-averaged variations of the C-index. For example, in multi-label classification problems, similar (see e.g. \citet{wu2017unified} with slightly different normalization) performance evaluation measures are referred to as the macro-AUC and instance-AUC, while the name micro-AUC is reserved for the constraint $\mathcal{R}^{\textnormal{\Cutil}}$.

Intuitively, the estimator as per Definition~\ref{genericestimatordef} average the utility value over all possible ways it can be evaluated on the sequence of data by reordering it, given that the reordering belongs to the utility functions domain $\mathcal{R}$. If the domain is not restricted, that is, it consists of all possible data sequences 
of length $\abs{\kernelf}$, the estimator reduces to the classical U-statistic \citep{hoeffding1948class}, the minimum variance unbiased estimators of $\theta$. However, with restricted domains, the estimator is not necessarily unbiased but still has certain desirable asymptotic properties as elaborated in the the following remark.
\begin{remark}
Since $\theta\in[0,1]$, its estimator per Definition~\ref{genericestimatordef} can be shown to have the following asymptotic properties. For any $\epsilon,\delta>0$ and $\operatorname{P}_Z\in\mathcal{P}$ with $\operatorname{P}_{\bm{Z}}[\mathcal{R}]>0$, the inequalities
\[
\operatorname{E}_{\bm{S}}[\abs{\widehat{\theta}(\bm{S})-\theta}]<\epsilon
\]
and
\[
\operatorname{P}_{\bm{S}}[\abs{\widehat{\theta}(\bm{S})-\theta}>\epsilon]<\delta
\]
hold when the sample size $\abs{\bm{S}}$ is large enough. In the literature, these are known as the \textbf{asymptotic unbiasedness} and \textbf{consistency}, respectively. The speed of convergence naturally depends on $\operatorname{P}_{\bm{Z}}[\mathcal{R}]$. For example, the estimator for AUC (see Example~\ref{cindexandaucexample}) converges slowly for very imbalanced binary class distributions. For more in depth analysis of the properties of estimators of (univariate) functions' conditional expectations, we refer to \citet{grunewalder2018plug}. Their work focuses especially on the uniform convergence property, indicating that the convergence takes place simultaneously over a possibly large set of estimators (e.g. estimators associated with some subsets of possible predictors $\mathbb{R}^{\mathcal{X}}$). The property is useful for optimizing prediction performance when designing learning algorithms. The scope of this paper is mainly on prediction performance estimation, and hence we analyze this no further.
\end{remark}

\section{Interaction Concordance}\label{secAindex}

We define the main and interaction effects, as well as collections of predictors that differ in their ability to model these effects in Section~\ref{secEffects}. Next, in Section~\ref{secICI} we define the concept of interaction concordance, an indicator of agreement between predicted and existing interaction directions. Finally, in Section~\ref{secComputational} we discuss how interaction concordance, as well as other previously considered performance estimators, can be calculated in a computationally efficient manner.

\subsection{Main and Interaction Effects}\label{secEffects}

We start by presenting the $2\times 2$ design formed by two drugs $\drug, \drug'$ and two targets $\targ,\targ^*$ as well as of the affinity strength measurements associated with the four drug-target combinations. The design can be expressed as
\begin{align}\label{dataquadruple}
z^q=\left(\begin{aligned}
z  && z^*\\%[1.5ex]
z' &&z'^*\end{aligned}
\right)\in\mathcal{Z}^{2\times 2},
\end{align}
where
\begin{align*}
\begin{aligned}
z &=(\drug,\targ, y)  &&\phantom{WW}&z^* &=(\drug,\targ^*,y^*)\\%[1.5ex]
z'&=(\drug',\targ,y') &&\phantom{WW}&z'^*&=(\drug',\targ^*,y'^*)\end{aligned}\;.
\end{align*}
Here, the data on the same row share the drugs, and those on the same column share the targets.

Next, we provide exact definitions for the effects in a single $2\times 2$ design with both observed and predicted affinity strength values.
\begin{definition}[Grand mean, drug main, target main, and interaction effects in $2\times 2$ design]\label{2x2designeffects}
Consider a $2\times 2$ design (\ref{dataquadruple}) formed from arbitrarily chosen two drugs and two targets, as well as their corresponding outputs drawn from some unknown distribution of data. Let us denote
\begin{align*}
x^q=\left(\begin{aligned}
(\drug,\targ)  &&(\drug,\targ^*)\\%[1.5ex]
(\drug',\targ) &&(\drug',\targ^*)\end{aligned}
\right)\in\mathcal{X}^{2\times 2}
\phantom{WWW}
y^q=\left(\begin{aligned}
y  &&y^*\\%[1.5ex]
y' &&y'^*\end{aligned}
\right)\in\mathbb{R}^{2\times 2}.
\end{align*}
Consider the decomposition of the quadruple of outputs $y^q$ into the following four orthogonal terms:
\begin{align*}
\begin{aligned}
=&y_C\cdot\left(\begin{aligned}
 1&&1\\%[1.5ex]
1&&1\end{aligned}
\right)
+y_\Drug\cdot\left(\begin{aligned}
 1&&1\\%[1.5ex]
-1 &&-1\end{aligned}
\right)
+y_\Targ\cdot\left(\begin{aligned}
1 &&-1\\%[1.5ex]
1 &&-1\end{aligned}
\right)
+y_{\Drug\times\Targ}\cdot\left(
\begin{aligned}
1 &&-1\\%[1.5ex]
-1 &&1
\end{aligned}
\right)
\end{aligned}\;,
\end{align*}
where
\begin{align*}
y_C&=\frac{1}{4}\left(y+y'+y^*+y'^*\right)\\
y_\Drug&=\frac{1}{4}\left(y-y'+y^*-y'^*\right)\\
y_\Targ&=\frac{1}{4}\left(y+y'-y^*-y'^*\right)\\
y_{\Drug\times\Targ}&=\frac{1}{4}\left(y-y'-y^*+y'^*\right)
\end{align*}
are referred to, respectively, as the \textbf{grand mean}, \textbf{drug main}, \textbf{target main} and \textbf{interaction} effects of $\drug$ and $\targ$ on $y$ at $z^q$.

Similarly, we decompose the predictor $\predfun$ on $x$ as
\begin{equation*}
\predfun(\drug,\targ)=\predfun_C+\predfunD(\drug)+\predfunT(\targ)+\predfunDT(\drug,\targ)\;,
\end{equation*}
whose terms can be expressed as
\begin{align*}
\predfun_C&=\frac{1}{4}\left(\predfun(\drug,\targ)+\predfun(\drug',\targ)+\predfun(\drug,\targ^*)+\predfun(\drug',\targ^*)\right)\\
\predfunD(\drug)&=\frac{1}{4}\left(\predfun(\drug,\targ)-\predfun(\drug',\targ)+\predfun(\drug,\targ^*)-\predfun(\drug',\targ^*)\right)\\%=-\predfunD(\drug')\\
\predfunT(\targ)&=\frac{1}{4}\left(\predfun(\drug,\targ)+\predfun(\drug',\targ)-\predfun(\drug,\targ^*)-\predfun(\drug',\targ^*)\right)\\%=-\predfunT(\targ^*)\\
\predfunDT(\drug,\targ)&=\frac{1}{4}\left(\predfun(\drug,\targ)-\predfun(\drug',\targ)-\predfun(\drug,\targ^*)+\predfun(\drug',\targ^*)\right)\;.
%=-\predfunDT(\drug',\targ)
%=-\predfunDT(\drug,\targ^*)
%=\predfunDT(\drug',\targ^*)\;.
\end{align*}
We say that these terms, respectively, model the grand mean, drug main, target main and interaction effects of the drugs and targets on the affinity values at the $2\times 2$ design $z^q$.
\end{definition}

We give the following names for specific collections of predictors based on their decomposition on a given $2\times 2$ design of drug-target pairs.
\begin{definition}\label{constancy_and_linearity}
Let $x^q$ be a $2\times 2$ design of drug-target pairs as per Definition~\ref{2x2designeffects}.
We say that $\predfun$ on $x^q$ is:
\begin{itemize}
\item \textbf{constant} if $\predfun(\drug,\targ)=\predfun(\drug',\targ)=\predfun(\drug,\targ^*)=\predfun(\drug',\targ^*)$;
\item \textbf{target symmetric} if $\predfun(\drug,\targ)=\predfun(\drug,\targ^*)$ and  $\predfun(\drug',\targ)=\predfun(\drug',\targ^*)$;
\item \textbf{drug symmetric} if $\predfun(\drug,\targ)=\predfun(\drug',\targ)$ and  $\predfun(\drug,\targ^*)=\predfun(\drug',\targ^*)$;
\item \textbf{additively separable} if $\predfun(\drug,\targ)+\predfun(\drug',\targ^*)=\predfun(\drug',\targ)+\predfun(\drug,\targ^*)$; and
\item \textbf{nonadditive} otherwise.
\end{itemize}
\end{definition}
It is straightforward to see that additively separable predictors miss the interaction term $\predfunDT$, and drug symmetric and target symmetric predictors additionally miss the terms $\predfunT$ and $\predfunD$, respectively, while constant predictors can only have the term $\predfun_C$.

\subsection{Interaction Concordance}\label{secICI}

We now introduce the concept of \textbf{interaction concordance} that indicates, for a quadruple of data $z^q$, whether the direction of interaction predicted by $\predfun$ agrees with that of the outputs $y^q$. In addition, we define \textbf{interaction concordance index} (\Aindex), a performance estimator designed for evaluating predictors for interaction prediction on a sequence of data. \Aindex assesses how well the predictor captures potential nonadditive interaction effects of drugs and targets on their affinities, while intentionally disregarding the accuracy of predictions based solely on the grand mean or either of the main effects.

Assume we are facing a choice between assigning $\drug$ to $\targ$ and $\drug'$ to $\targ^*$ or the reverse assignment $\drug'$ to $\targ$ and $\drug$ to $\targ^*$. This choice can be quantified by considering the direction of interaction on $z^q$. The first assignment is preferable if $y+y'^*>y^*+y'$. The quadruples of data, for which this condition holds, form the following subset of $\mathcal{Z}^4$:
\[
\mathcal{R}^{\textnormal{\Autil}}=\left\{
\left(\begin{array}{cc}
(\drug,\targ, y)&\phantom{i}(\drug,\targ^*,y^*)\\[1.5ex]
(\drug',\targ,y')&\phantom{i}(\drug',\targ^*,y'^*)
\end{array}\right)
\mid 
(\drug,\drug',\targ,\targ^*)\in\mathcal{D}^2\times\mathcal{T}^2, y+y'^*>y^*+y'
\right\}\;,
\]
that is, the data on the same row of the quadruple share the drugs and those on the same column share the targets, and the last condition indicates that the aggregate affinity strength of the first assignment is greater than that of the reverse.

To indicate whether the interaction direction predicted by $\predfun$ agrees with that determined by the outputs, we define the following degree 4 utility function:
\begin{definition}[Interaction concordance]\label{icutildef}
The correctness of the interaction direction's prediction can be expressed as
\[
\kernelf_\predfun^{\textnormal{\Autil}}
(z,z^*,z',z'^*)=H\left(\predfun(x)-\predfun(x^*)-\predfun(x')+\predfun(x'^*)\right)
\]
restricted to the set  $\mathcal{R}^{\textnormal{\Autil}}$.
\end{definition}

Analogously to the above-considered utilities (see Definition \ref{genericestimand} and Example \ref{concordanceexample2}), the probability distribution counterpart of interaction concordance can be expressed as the value of the conditional expectation functional:
\[
\theta_\predfun^{\textnormal{\Autil}}=\operatorname{E}\left[H\left(f(X)-f(X^*)-f(X')+f(X'^*)\right)\middle|\left(\begin{array}{cc}
Z&\phantom{W}Z^*\\[1.5ex]
Z'&\phantom{W}Z'^*
\end{array}\right)\in\mathcal{R}^{\textnormal{\Autil}}\right]\;,
\]
where the expectation is taken over a distribution of data $\operatorname{P}_Z$, such that $\operatorname{P}_{\bm{Z}}[\mathcal{R}^{\textnormal{\Autil}}]>0$.
Finally, for a given sample of data $\bm{s}$, we obtain the following estimator that averages the utility over all $2\times 2$ designs that can be formed from the sample:
\begin{definition}[Interaction concordance index]\label{defAindex}
Let $\bm{s}\in \mathcal{Z}^{\abs{\bm{s}}}$ be a  sample of data. We refer to the following estimator of $\theta_\predfun^{\textnormal{\Autil}}$ as \Aindex:
\[
\widehat{\theta}^{\textnormal{\Aindex}}_\predfun(\bm{s})=\abs{\mathcal{I}}^{-1}
\sum_{\sigma
\in\mathcal{I}
}
H(\predfun(x_{i})-\predfun(x_{i^*})-\predfun(x_{i'})+\predfun(x_{i'^*}))\;,
\]
where
\[
\mathcal{I}=\left\{\sigma=\left(\begin{array}{cc}
i& i^*\\
i'& i'^*
\end{array}\right)\middle|\sigma\cdot\bm{s}\in\mathcal{R}^{\textnormal{\Autil}} \right\}\;.
\] 
\end{definition}

\smallskip
\noindent
The next proposition follows from Definitions  \ref{constancy_and_linearity} and \ref{defAindex}.
\begin{proposition} Interaction concordance, and thereby also its distribution level counterpart and IC-index, is invariant to additions of constant, drug symmetric, target symmetric and additively separable functions.
\end{proposition}
\begin{proof} Let $\predfun$ be a predictor, $\predfun_{\mathcal{D+T}}$ an additively separable function, and $(z,z^*,z',z'^*)
\in \mathcal{R}^{\textnormal{IC}}$ 
a quadruple of data. We recall that $\predfun_{\mathcal{D+T}} = \predfunD + \predfunT+ \predfun_C $, where $\predfunD$, $\predfunT$, and $\predfun_C$ are the components depending on only drug, only target, and neither drug nor target, respectively. We first show that $\kernelf^{\textnormal{IC}}_{\predfun}(z,z^*,z',z'^*) = \kernelf^{\textnormal{IC}}_{\predfun+\predfun_{\mathcal{D+T}}}(z,z^*,z',z'^*)$ for all additively separable functions $\predfun_{\mathcal{D+T}}$:
\begin{align*}
    \kernelf^{\textnormal{IC}}_\predfun(z,z^*,z',z'^*) 
    &= H(f(d,t)-f(d,t^*)-f(d',t)+f(d',t^*)) \\
    &= H(f(d,t)+\predfun_{\mathcal{D+T}}(\drug,\targ)-f(d,t^*)-\predfun_{\mathcal{D+T}}(\drug,\targ^*) \\
    & \qquad -f(d',t)-\predfun_{\mathcal{D+T}}(\drug',\targ)+f(d',t^*)+\predfun_{\mathcal{D+T}}(\drug',\targ^*)) \\
    &=\kernelf^{\textnormal{IC}}_{\predfun+\predfun_{\mathcal{D+T}}}(z,z^*,z',z'^*),
\end{align*}
since  
\begin{align*}
\predfun_{\mathcal{D+T}}(\drug,\targ)+\predfun_{\mathcal{D+T}}(\drug',\targ^*) &= (\predfunD(\drug)+\predfunT(\targ)+\predfun_C) + (\predfunD(\drug')+\predfunT(\targ^*)+\predfun_C)\\
&= \predfun_{\mathcal{D+T}}(\drug,\targ^*) + \predfun_{\mathcal{D+T}}(\drug',\targ).\end{align*}
The invariancy to constant, drug symmetric, and target symmetric functions directly follows.
\end{proof}
For the other utilities considered so far, analogous invariances are summarized in Table~\ref{tab:predictorutilityinvariances}.

\begin{table}[h]
\begin{center}
{\small
    \begin{tabular}{l|ccccc} \hline
\hline \\
[-2.0ex] $\predfun$\hfill\textbackslash\hfill $\kernelf$&$\kernelf_\predfun^{\textnormal{Acc}}$&$\kernelf_\predfun^{\textnormal{\Cutil}}$&$\kernelf_\predfun^{\textnormal{\CDutil}}$&$\kernelf_\predfun^{\textnormal{\CTutil}}$&$\kernelf_\predfun^{\textnormal{\Autil}}$\\
[1.0ex]
\hline \\ [-2.0ex]
Zero&0.5&0.5&0.5&0.5&0.5\\
Constant&&0.5&0.5&0.5&0.5\\
Drug symmetric &&&&0.5&0.5\\
Target symmetric &&&0.5&&0.5\\
Additively separable&&&&&0.5\\
Nonadditive&&&&&\\
\hline
\hline
\end{tabular}
\caption{Summary of invariances of the considered utility functions with respect to the types of predictors named in Definition~\ref{constancy_and_linearity}. The acronyms and definitions of the considered utilities are accuracy (Acc) per Definition~\ref{def:accuracy}, concordance (\Cutil) per Definition~\ref{concordanceindicatordef}, drugwise concordance (\CDutil) per Definition~\ref{drugntargetwisecidef}, targetwise concordance (\CTutil) per Definition~\ref{drugntargetwisecidef} and interaction concordance (\Autil) Definition~\ref{icutildef}.
The table cell is 0.5 if it is the only possible utility value for any data.}
\label{tab:predictorutilityinvariances}
}
\end{center}
\end{table}

\subsection{Estimators' Computational Complexity}\label{secComputational}
Here, we briefly analyze the computational complexity of the considered estimators. In what follows, we denote by
$\drsize = \abs{\mathcal{D}_{\bm{z}}}$ the number of unique drugs and $\tasize=\abs{\mathcal{T}_{\bm{z}}}$ the number of unique targets. For the sake of simplicity, we assume the case where $\trsize\simeq\drsize\tasize$, that is, most of the drug-target pairs have a known DTA value. 

Univariate performance measures, like the mean squared error and the classification accuracy, can be calculated in $O(\drsize\tasize)$  on a sample of data: they decompose into a sum where the loss is evaluated for each pair exactly once. However, a direct approach to estimate the global rank correlation by iterating over all drug-target quadruples to count concordant and discordant pairs results in a complexity of $O(\drsize^2 \tasize^2)$. For the drugwise and targetwise rank correlations and for the \Aindex, the complexities of running such an algorithm are  $O(\drsize\tasize^2)$, $O(\drsize^2\tasize)$, and $O(\drsize^2\tasize^2)$, respectively. These complexities are impractical for large data sets. Notably, \citet{newson2006efficient} propose an algorithm that calculates numbers of concordant and discordant pairs for a sample of $n$ observations in $O(n\log(n))$ time by performing  $O(n\log(n))$ complexity sorting operation, $O(n)$ logarithmic time insertion, and search operations on a binary search tree. This approach directly improves the complexities of global, drugwise, and targetwise ranking error evaluations to $O(\drsize\tasize\log(\drsize\tasize))$, $O(\drsize\tasize\log(\tasize))$, and $O(\drsize\tasize\log(\drsize))$, respectively. Further, as a straightforward extension, the approach allows calculating the \Aindex in $O(\min(\drsize^2\tasize\log(\tasize),\drsize\tasize^2\log(\drsize))$ time, as it reduces to calculating the number of concordant and discordant target pair differences for each possible combination of drugs (or vice versa).

\section{Prediction Performance of Learning Algorithms}

Here, we move from prediction performance estimation of fixed predictors to that of learning algorithms. In Section~\ref{algutilitysection}, we present utility functions for learning algorithms in four different prediction problems based on the off-training set partition as per Definition~\ref{def:otspartition}. Their properties are analyzed in Section~\ref{equivariancesection}.

\subsection{Utilities for Learning Algorithms}
\label{algutilitysection}
Here we focus on utilities of learning algorithms' prediction performance. They are analogous to those of predictors, except that they also account for the learning algorithm and the training data the predictor is inferred from. To stress the distinction between learning algorithms' and predictors' utilities, as well as the nested nature of the former, we may refer to $\akernelf$ as the \textbf{outer utility} and to $\kernelf$ as the \textbf{inner utility} of $\akernelf$. The following definition leans on inner utilities listed in Table~\ref{tab:predictorutilityinvariances}, but is by no means restricted to only those.
\begin{definition}[Utility for learning algorithms' prediction performance]
\label{algoutilitydef}
Let $\kernelf$ be any of the inner utilities summarized in Table~\ref{tab:predictorutilityinvariances}, and let $\abs{\kernelf}$ and $\mathcal{R}$ denote its degree and domain, respectively. For a sequence $\bm{z}$ of data of length $\trsize+\abs{\kernelf}$, let $\bm{z}^{\textnormal{Train}}$ and $\bm{z}^\textnormal{Test}$ denote, respectively, the sequences consisting of the first $\trsize$ and the remaining $\abs{\kernelf}$ entries of $\bm{z}$:
\begin{align}\label{trainestsplit}
\bm{z}=(\underbrace{z_1,\ldots,z_\trsize}_{\bm{z}^\textnormal{Train}},\underbrace{z_{\trsize+1},\ldots,z_{\trsize+\abs{\kernelf}}}_{\bm{z}^\textnormal{Test}})\;.
\end{align}
Let $\algo$ be a learning algorithm as per Definition~\ref{algodef}. Moreover, let
\[
\predictorvar_{\bm{z}^\textnormal{Train}}
=\algo\left(\bm{z}^\textnormal{Train}\right)
\]
be the random element of the predictor learned from $\bm{z}^\textnormal{Train}$ by $\algo$ and let $\predfun_{\bm{z}^\textnormal{Train}}$ denote its realization. Then, let
\begin{align}\label{holdoutdef}
\akernelf_\algo(\bm{z})
=\kernelf_{\predfun_{\bm{z}^\textnormal{Train}}}\left(\bm{z}^\textnormal{Test}\right)\;
\end{align}
denote an outer utility of degree $\abs{\akernelf}=\trsize+\abs{\kernelf}$ for a learning algorithm $\algo$ on $\bm{z}$. The value of $\akernelf_\algo(\bm{z})$ indicates how well $\predfun_{\bm{z}^\textnormal{Train}}$ predicts the outputs of $\bm{z}^{\textnormal{Test}}$ in terms of the inner utility $\kernelf$. To account for the possible restriction $\mathcal{R}$ on the domain of $\kernelf$, we pose the corresponding restriction for the domain of $\akernelf_\algo$:
\begin{equation}\label{genperfrestriction}
\mathcal{C}=\left\{\bm{z}\in\mathcal{Z}^{\trsize+\abs{\kernelf}}\mid\bm{z}^\textnormal{Test}\in\mathcal{R}\right\}\;.
\end{equation}
\end{definition}
Note that we here condition the size of training data to be $\trsize$, because prediction performance tends to be strongly dependent on it for the most of the practically relevant learning algorithms. In practice, and also in our experimental evaluations, we may relax this constraint by allowing $\trsize$ to be inside a given interval rather than a single number.

The next example presents the interaction concordance of learning algorithm obtained by substituting the fixed predictors' interaction concordance $\kernelf^{\textnormal{\Autil}}$ and its domain to the inner utility of $\akernelf$.
\begin{example}[Learning algorithm's interaction concordance]
\label{algo_ic_example}
Let $\algo$ be a learning algorithm and let $\kernelf_\predfun^{\textnormal{\Autil}}$ and $\mathcal{R}^{\textnormal{\Autil}}$ be the interaction concordance and its domain, respectively. Then, the interaction concordance of a learning algorithm $\algo$ can be expressed as
\begin{align*}
\akernelf_\algo^{\textnormal{\Autil}}(\bm{z})
=\kernelf_{\predfun_{\bm{z}^\textnormal{Train}}}^{\textnormal{\Autil}}\left(\bm{z}^\textnormal{Test}\right)\;,
\end{align*}
where $\predfun_{\bm{z}^\textnormal{Train}}$ is a realization of $\predictorvar_{\bm{z}^\textnormal{Train}}=\algo\left(\bm{z}^\textnormal{Train}\right)$ and whose domain is
\[
\mathcal{C}^{\textnormal{\Autil}}
=\left\{\bm{z}\in\mathcal{Z}^{\trsize+\abs{\kernelf}}\mid\bm{z}^{\textnormal{Test}}\in\mathcal{R}^{\textnormal{\Autil}}\right\}\;.
\]
\end{example}

In machine learning literature, quantities of type $\akernelf$ are often used to analyze learning algorithms' generalization performance. In contrast, this paper mainly focuses on their OTS prediction performance, that is, performance on data, whose inputs are distinct from those present in the training data (see e.g. \citet{wolpert1996lack,roos2005generalization}). 
This can be formalized by tightening the restriction (\ref{genperfrestriction}) to
\begin{align}\label{basicotsrestrictioneq}
\mathcal{C}^{\textnormal{OTS}}=\left\{\bm{z}\in\mathcal{Z}^{\trsize+\abs{\kernelf}}\mid\bm{z}^{\textnormal{Test}}\in\mathcal{R},\mathcal{X}_{\bm{z}^{\textnormal{Test}}}\subseteq\mathcal{X}\setminus\mathcal{X}_{\bm{z}^{\textnormal{Train}}}
\right\}\;.
\end{align}
In particular, we analyze the performance on OTS data along the partition given in Definition~\ref{def:otspartition}. The restrictions corresponding to these cases are presented in the following explicit definition.
\begin{definition}[Four types of off-training-set utilities]\label{settingrestrictionsdef}
Let $\kernelf$, $\mathcal{R}$ and $\akernelf$ be as in Definition~\ref{algoutilitydef}. Then, by restricting the domain of $\akernelf$ as
\begin{align}
\label{conditionsets}
\begin{aligned}
\mathcal{C}^{\textnormal{IDIT}} &=\left\{\bm{z}\in\mathcal{Z}^{\trsize+\abs{\kernelf}}\mid\bm{z}^{\textnormal{Test}}\in\mathcal{R},\mathcal{X}_{\bm{z}^{\textnormal{Test}}}\subseteq \Xkdkt_{\bm{z}^{\textnormal{Train}}}
\right\}\\
\mathcal{C}^{\textnormal{ODIT}} &=\left\{\bm{z}\in\mathcal{Z}^{\trsize+\abs{\kernelf}}\mid\bm{z}^{\textnormal{Test}}\in\mathcal{R},\mathcal{X}_{\bm{z}^{\textnormal{Test}}}\subseteq \Xndkt_{\bm{z}^{\textnormal{Train}}}
\right\}\\
\mathcal{C}^{\textnormal{IDOT}} &=\left\{\bm{z}\in\mathcal{Z}^{\trsize+\abs{\kernelf}}\mid\bm{z}^{\textnormal{Test}}\in\mathcal{R},\mathcal{X}_{\bm{z}^{\textnormal{Test}}}\subseteq \Xkdnt_{\bm{z}^{\textnormal{Train}}}
\right\}\\
\mathcal{C}^{\textnormal{ODOT}} &=\left\{\bm{z}\in\mathcal{Z}^{\trsize+\abs{\kernelf}}\mid\bm{z}^{\textnormal{Test}}\in\mathcal{R},\mathcal{X}_{\bm{z}^{\textnormal{Test}}}\subseteq \Xndnt_{\bm{z}^{\textnormal{Train}}}
\right\}
\end{aligned}
\;,
\end{align}
we get the utilities corresponding to the four OTS prediction performances along the partition of OTS data as per Definition~\ref{def:otspartition}.
\end{definition}
We continue Example~\ref{algo_ic_example} by posing the additional restriction $\mathcal{C}^{\textnormal{IDIT}}$ on the outer utility's domain. Accordingly, we present the variant of learning algorithms' interaction concordance that focuses on IDIT data:
\begin{example}[Learning algorithm's interaction concordance on IDIT data]
\label{ic-idit_example}
Let $\algo$ be a learning algorithm and let $\kernelf_\predfun^{\textnormal{\Autil}}$ and $\mathcal{R}^{\textnormal{\Autil}}$ be the interaction concordance and its domain, respectively. Then, the interaction concordance of a learning algorithm $\algo$ on off-training-set data with in-training-set drugs and in-training-set-targets (IC-IDIT) can be expressed as
\begin{align*}
\akernelf_\algo^{\textnormal{\Autil-IDIT}}(\bm{z})
=\kernelf_{\predfun_{\bm{z}^\textnormal{Train}}}^{\textnormal{\Autil}}\left(\bm{z}^\textnormal{Test}\right)\;,
\end{align*}
where $\predfun_{\bm{z}^\textnormal{Train}}$ is a realization of $\predictorvar_{\bm{z}^\textnormal{Train}}=\algo\left(\bm{z}^\textnormal{Train}\right)$ and whose domain is
\[
\mathcal{C}^{\textnormal{\Autil-IDIT}}
=\left\{\bm{z}\in\mathcal{Z}^{\trsize+\abs{\kernelf}}\mid\bm{z}^{\textnormal{Test}}\in\mathcal{R}^{\textnormal{\Autil}},\mathcal{X}_{\bm{z}^{\textnormal{Test}}}\subseteq \Xkdkt_{\bm{z}^{\textnormal{Train}}}
\right\}\;.
\]
\end{example}
\noindent The IDOT, ODIT and ODOT variants are defined analogously. Similar variants of $\akernelf$ can also be defined for other inner utilities given in Table~\ref{tab:predictorutilityinvariances}.

Analogous to the utilities' distributional counterparts considered in Definition~\ref{genericestimand}, we define the distribution level utilities for learning algorithms:
\begin{definition}[Distribution utility of learning algorithm]\label{otsestimand}
Let $\algo$ and $\akernelf$ be as in Definition~\ref{algoutilitydef}. Then, the expected utility of $\algo$ is
\begin{align*}
\theta_\algo
&=\operatorname{E}\left[\akernelf_\algo(\bm{Z})\mid \bm{Z}\in\mathcal{C}\right]\;,
\end{align*}
where $\mathcal{C}$ is any of the restrictions given in Definition~\ref{settingrestrictionsdef} and the expectation is taken over probability distributions on data such that $\operatorname{P}_{\bm{Z}}[\mathcal{C}]>0$. If the learning algorithm $\algo$ is randomized, the expectation also accounts for the distribution of the random element $\predictorvar_{\bm{Z}}$.
\end{definition}

Given a sample $\bm{s}$ of IID data drawn from $\operatorname{P}_{Z}$, we obtain estimators of $\theta_\algo$ by averaging over the different possibilities of evaluating $\akernelf_\algo$ on the sample that are conformable with $\mathcal{C}$, similarly to Definition~\ref{genericestimatordef}. However, each evaluation of the utility requires retraining and the number of possibilities is combinatorial with respect to the sample size, which makes it computationally infeasible in most practical cases. Therefore, one usually resorts to incomplete estimators only averaging over a small and assumedly representative number of train-test splits of the sample, popularly known as cross-validation (CV) or repeated hold-out estimators. A typical CV estimator can be expressed as
\begin{align}\label{cvestimator}
\widehat{\theta}_\algo(\bm{s})
&=\abs{\Upsilon}^{-1}
\sum_{\pi\in\Upsilon}\widehat{\theta}_{\kernelf_{\predfun_{\bm{s}^\textnormal{Train}}}}(\bm{s}^{\textnormal{Test}})\;,
\end{align}
where $\Upsilon$ is some subset of permutations $\pi\cdot\bm{s}$ of the sample sequence, such that
\begin{align}\label{samplesplit}
\pi\cdot\bm{s}=(\underbrace{z_1,\ldots,z_\trsize}_{\bm{s}^\textnormal{Train}},\underbrace{z_{\trsize+1},\ldots,z_{\trsize+\testsetsize}}_{\bm{s}^\textnormal{Test}},\underbrace{z_{\trsize+\testsetsize+1},\ldots,
z_{\abs{\bm{s}}}}_{\bm{s}^\textnormal{Ignored}})\;,
\end{align}
for some $\testsetsize\leq\abs{\bm{s}}-\trsize$.
Similarly to (\ref{trainestsplit}), the part $\bm{s}^\textnormal{Train}$ denotes the data used for inferring the predictor in the CV round. Here $\bm{s}^\textnormal{Test}$ denotes the data used for estimating the prediction performance of the predictor inferred from the training data, given one of the restrictions in (\ref{conditionsets}).  Note that while $\bm{z}^\textnormal{Test}$ in~(\ref{trainestsplit}) is defined to be exactly of length $\abs{\kernelf}$, here the length of $\bm{s}^\textnormal{Test}$ is not fixed but depends on how large portion of the remaining sample is conformable with the restriction. The part $\bm{s}^\textnormal{Ignored}$ consists of the sample data unusable for testing due to the restriction and has to be ignored in the cross-validation round.
One split of type (\ref{samplesplit}) is illustrated in Figure~\ref{fig:testtrain}. In the figure, a sample of data $\bm{s}$ consists of the drug-target indexed matrix's entries colored with black, blue, yellow, beige and red. The black colored entries refer to the data in  $\bm{s}^\textnormal{Train}$ and the remaining parts of the sample belong to either $\bm{s}^\textnormal{Test}$ or  $\bm{s}^\textnormal{Ignored}$ depending which of the restrictions per Definition~\ref{settingrestrictionsdef} are in place.

\begin{figure}
    %\centering 
    \begin{tcolorbox}
    \centering
    \includegraphics[width=0.950\linewidth]{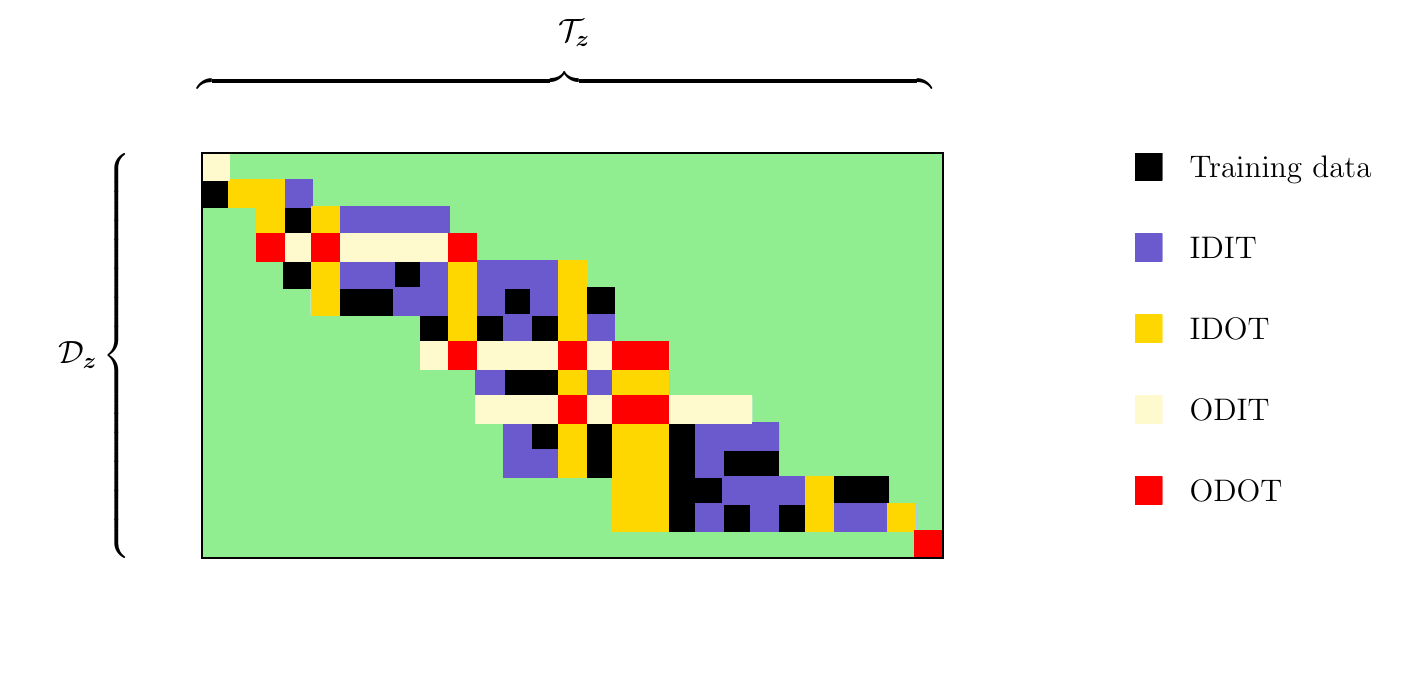}
    \end{tcolorbox}
    \caption{An example of training-test splits corresponding to different OTS problems: in all problems, black squares denote the training set. The test set for the IDIT problem may consist of any blue datum, while test sets IDOT and ODIT problems may contain any dark yellow and light yellow datum, respectively. Finally, any red datum may be included in the test set simulating the ODOT problem.}
    \label{fig:testtrain}
\end{figure}

\subsection{Permutation Equivariance and Side Information}\label{equivariancesection}

We now focus on analysing whether different types of learning algorithms can solve learning problems in terms of the above defined utilities. For this purpose, we nail down the following terms.
\begin{definition}[Learning problems associated to utility]\label{def:learningproblem}
Let $\kernelf$, $\mathcal{R}$ and $\akernelf$ be as in Definition~\ref{algoutilitydef}, and let $\mathcal{C}$ be one of the four  restrictions of $\akernelf$ as per Definition~\ref{settingrestrictionsdef}. Let $\mathcal{P}$ be the collection of probability distributions on data such that $\operatorname{P}_{\bm{Z}}[\mathcal{C}]>0$ for all $\operatorname{P}_Z\in\mathcal{P}$. Consider the estimand $\theta_\algo$ as per  Definition~\ref{otsestimand} as a functional $G:\operatorname{P}_{Z},\algo\mapsto\theta_\algo$ on data distributions and learning algorithms, that is
\begin{align*}
G(\operatorname{P}_{Z},\algo)&=\operatorname{E}\left[\akernelf_\algo(\bm{Z})\mid \bm{Z}\in\mathcal{C}\right]\;,
\end{align*}
where the expectation is taken over $\operatorname{P}_{Z}$ and the distribution of $\predictorvar_{\bm{Z}}$. We refer to $G=G(\cdot,\cdot)$ as the \textbf{collection of learning problems associated with utility $\akernelf$}, and any of its members $G(\operatorname{P}_{Z},\cdot)$ as a \textbf{learning problem associated with $\akernelf$}.
\end{definition}

We continue Example~\ref{ic-idit_example} and provide
the collection of learning problems associated with learning algorithm's interaction concordance on IDIT data. Analogous collections associated with combinations of the five inner utilities in Table~\ref{tab:predictorutilityinvariances} and the four outer ones given in Definition~\ref{settingrestrictionsdef} are formed similarly.
\begin{example}[IC-IDIT learning problems]\label{IC-IDIT_problems_example}
Let $\akernelf^{\textnormal{\Autil-IDIT}}$ and $\mathcal{C}^{\textnormal{\Autil-IDIT}}$ be as in Example~\ref{ic-idit_example} and let $\mathcal{P}$ be the collection of probability distributions on data such that $\operatorname{P}_{\bm{Z}}[\mathcal{C}^{\textnormal{\Autil-IDIT}}]>0$ for all $\operatorname{P}_Z\in\mathcal{P}$. Then, the  functional \[G^{\textnormal{\Autil-IDIT}}:\operatorname{P}_{Z},\algo\mapsto\theta_\algo
\]
can be considered as the collection of learning problems associated with the learning algorithms' interaction concordance on IDIT data $\akernelf^{\textnormal{\Autil-IDIT}}$.
\end{example}

To analyze whether different types of learning algorithms are suitable for solving the above considered learning problems, we define the concept of \textbf{alignment} between them.
\begin{definition}[Alignment of a learning algorithm with learning problems]
Let $\akernelf$, $\mathcal{C}$, $\mathcal{P}$, $G(\operatorname{P}_{Z},\cdot)$ and $G(\cdot,\cdot)$ be as in Definition~\ref{def:learningproblem}. If $G(\operatorname{P}_{Z},\algo)\leq 0.5$, we say that learning algorithm $\algo$ is \textbf{badly aligned} with the learning problem $G(\operatorname{P}_{Z},\cdot)$, that is, the expected utility is either exactly at random level or even worse. Otherwise, $\algo$ is \textbf{well-aligned} with the learning problem.
In particular, we say that $\algo$ is \textbf{uniformly badly aligned} with the collection $G(\cdot,\cdot)$ of learning problems associated with the utility $\akernelf$ if 
\begin{align}\label{uniformbadalignment}
G(\operatorname{P}_{Z},\algo)=0.5\quad\forall\operatorname{P}_{Z}\in\mathcal{P}\;.
\end{align}
\end{definition}
\begin{remark}
As a slightly related work, we recall the no-free-lunch theorem in supervised classification \citep{wolpert1996lack}. If one ``averages'' the expected OTS binary classification accuracy of a learning algorithm over all distributions of data $\operatorname{P}_{Z}\in\mathcal{P}$, one gets exactly 0.5. It is straightforward to show that this also concerns the other utilities considered in this paper. That is, on average, any learning algorithm is badly aligned.
\end{remark}
Next, we point out types of learning algorithms that are badly aligned with some of the considered collections of learning problems. Our first and most straightforward partition of learning algorithms follows the additively separable and nonadditive predictor collections in Definition~\ref{constancy_and_linearity}. This coincides with the classical categorizations of learning algorithms based on whether they infer linear or nonlinear models. The former type of learning algorithms are badly aligned with $G^{\textnormal{\Autil}}$.

One can similarly consider algorithms only able to learn either constant, drug-symmetric or target-symmetric predictors, but such learning algorithms are rather trivial in our context. However, they are still useful and practical as baselines and reference points in algorithm comparisons. Typical examples are learning algorithms inferring \textbf{majority classifiers} or \textbf{mean predictors}. In binary classification, a majority classifier predicts for any data the majority class of the training data. Mean predictor predicts the mean DTA value in the training data. Similarly, one can have their drugwise and targetwise counterparts that predict, for a drug-target pair, the majority or mean DTA value of training data having the same drug component and the majority or mean DTA of training data having the same target component, respectively. Obviously, these learning algorithms are badly aligned with all learning problems associated with utilities invariant with respect to such restricted model types, as shown in Table~\ref{tab:predictorutilityinvariances}. For example, the drugwise mean predictor is uniformly badly aligned with both $G^{\textnormal{\Autil}}$ and $G^{\textnormal{\CTutil}}$. For other inner utilities, the drugwise mean predictor shows that even this simple learning algorithm may obtain quite competitive results if the drug main effects are highly dominant compared to the other effects in the data. 

We now turn our focus on a more interesting categorization of learning algorithms based on their permutation equivariance properties with respect to drug identities and/or target identities (see e.g. \citet{bogatskiy2022symmetry,pan2022equivariant} and references therein for the use of equivariance in machine learning). Intuitively, these indicate whether learning algorithm's  inductive bias contains any systematic differences between drugs or between targets prior to training phase. Then, we analyze the effect of this type of inductive bias---or rather the absence of it---on the generalization to drugs, targets or both, that are not encountered in the training data. The equivariance properties in question are formally defined as follows.
\begin{definition}[Learning algorithms' permutation equivariance to drug and target identities]\label{equivariancedef}
Let $\operatorname{\Pi}_{\mathcal{D}}$ and $\operatorname{\Pi}_{\mathcal{T}}$ denote, respectively, the finitary symmetric groups on drug and target identities (i.e., they contain all finitary permutations for drugs and targets)\footnote{While the sets of drugs and targets could be infinite or even uncountable, we restrict our consideration to finitary permutations involving only a finite number of drug or target identity exchanges, since, in practice, we only encounter finite numbers of drugs and targets.}. The action of $\pi_{\mathcal{D}}\in\operatorname{\Pi}_{\mathcal{D}}$ on $\bm{z}$ is specified such that, if $(\drug_i,\targ_i,y_i)$ represents the $i$th entry of $\bm{z}$, then the $i$th entry of the image $\pi_{\mathcal{D}}\cdot\bm{z}$ is $(\pi_{\mathcal{D}}(\drug_i),\targ_i,y_i)$. Similarly, the actions of $\pi_{\mathcal{D}}$ on predictor $\pi_{\mathcal{D}}\cdot\predfun(\drug,\targ)=\predfun(\pi_{\mathcal{D}}(\drug),\targ)$. Let $\algo$ be a randomized learning algorithm as per Definition~\ref{algodef}. We say that $\algo$ is \textbf{permutation equivariant} with respect to drug identities, or shortly \textbf{drug permutation equivariant}, if
\begin{align*}
\operatorname{P}[\predictorvar_{\pi_{\mathcal{D}}\cdot\bm{z}}\in\mathcal{F}]=\operatorname{P}[\predictorvar_{\bm{z}}\in\pi_{\mathcal{D}}\cdot\mathcal{F}]
\end{align*}
holds for all sequences of data $\bm{z}$, all measurable sets of predictors $\mathcal{F}\subseteq\mathbb{R}^{\mathcal{X}}$ and all $\pi_{\mathcal{D}}\in\operatorname{\Pi}_{\mathcal{D}}$. Learning algorithm $\algo$ being \textbf{target permutation equivariant} is defined analogously. Intuitively, equivariance indicates that swapping the identities of drugs, targets, or both in the training data has the same effect as if the identities were swapped in the distribution of predictors learned from the data.
\end{definition}
Recall that the training data is defined as a sequence of affinity strength observations associated with categorical variable values for drugs and targets, but side information about the drugs, targets or both may also be available, as depicted in Figure~\ref{fig:sideinfo}. Note that we consider the side information to be independent of the training data, and hence it can be interpreted as being a part of the algorithm's inherent inductive bias. Conversely, we interpret the absence of side information in the widest possible sense, that is, the learning algorithm's inductive bias does not even implicitly make any other difference between drugs or between targets than considering them as distinct categorical values. Therefore, by the principle of indifference, no systematic prediction differences between them can take place that is not solely inferred from training data. 

We now present an exhaustive result for all learning problems considered in this paper.
\begin{proposition}[Alignment of drug and target permutation equivariant learning algorithms]
\label{alignmentproposition}
Consider the partition of learning problems in Table~\ref{alignment_table}, that contains 20 distinct collections of learning problems formed by composing one of the five inner utilities in Table~\ref{tab:predictorutilityinvariances} with one of the four OTS outer utilities as per Definition~\ref{settingrestrictionsdef}.
Learning algorithm's drug and target permutation equivariance implies that it is uniformly badly aligned with the collections of learning problems shaded with red in Table~\ref{alignment_table}. The implication does not hold for the other collections in the table.
\begin{table}[h]
\begin{center}
\usetikzlibrary{patterns}
{\upshape
\begin{NiceTabular}[corners]{|c|W{c}{0.8cm}W{c}{0.2cm}W{l}{1cm}:W{c}{0.8cm}W{c}{0.2cm}W{l}{1cm}|}
\CodeBefore
\rectanglecolor{yellow!35}{9-9}{11-11}
  \begin{tikzpicture}
	\fill [pattern=north west lines,pattern color=red] (6-|7) |- (8-|5) -- cycle ;
  \end{tikzpicture}
\Body
\Hline
& \Block{1-3}{ITS targets} &&& \Block{1-3}{OTS targets}&&\\
\Hline
\rotate\Block[v-center]{3-1}{ITS\\drugs} &
\Block[v-center]{3-3}{
\begin{tabular}{W{c}{0.7cm}W{c}{0.1cm}W{l}{1cm}}
Acc&& C$_{d}$
\\
& C &\\
C$_{t}$ && IC\\
\end{tabular}
}&&&
\Block[v-center]{3-3}{
\begin{tabular}{W{c}{0.7cm}W{c}{0.1cm}W{l}{1cm}}
Acc&& C$_{d}$
\\
& C &\\
C$_{t}$ && IC \\
\end{tabular}
}&&\Block[transparent,tikz={pattern = north west lines,pattern color=red}]{6-1}{}\\
&&&&&&\\
&&&&&&\\
&&&&&&\\
%\Hline
\hdottedline
\rotate\Block[v-center]{3-1}{OTS\\drugs}
&
\Block[v-center]{3-3}{
\begin{tabular}{W{c}{0.7cm}W{c}{0.1cm}W{l}{1cm}}
Acc&& C$_{d}$
\\
& C &\\
C$_{t}$ && IC \\
\end{tabular}
}&&&
\Block[v-center]{3-3}{
\begin{tabular}{W{c}{0.7cm}W{c}{0.1cm}W{l}{1cm}}
Acc&& C$_{d}$
\\
& C &\\
C$_{t}$ && IC\\
\end{tabular}
}
&&\\
&&&&&&\\
&\Block[transparent,tikz={pattern = north west lines,pattern color=red}]{1-6}{}&&&&&\\ 
\hline
\end{NiceTabular}
}
\end{center}
\caption{The table is first partitioned, with vertical and horizontal dotted lines, along the four restrictions as per Definition~\ref{settingrestrictionsdef} of the utility $\akernelf$ of the learning algorithm as per Definition~\ref{algoutilitydef}. Namely, top-left: in-training-set drugs and in-training-set targets (IDIT), top-right: in-training-set drugs and off-training-set targets (IDOT), bottom-left:
off-training-set drugs and in-training-set targets (ODIT), and bottom-right: off-training-set drugs and off-training-set targets (ODOT). These four areas are subsequently divided according to the five predictors' utilities with acronyms as in Table~\ref{tab:predictorutilityinvariances}: accuracy (ACC), concordance (\Cutil), drugwise concordance (\CDutil), targetwise concordance (\CTutil) and interaction concordance (\Autil).
Each of the 20 cells corresponds to collection of learning problems associated with a utility function $\akernelf$. For example, the slot IC in the IDIT area corresponds to $G^{\textnormal{\Autil-IDIT}}$ as per Example~\ref{IC-IDIT_problems_example}. The red shaded area illustrates the learning problems that drug and target permutation equivariant learning algorithms are uniformly badly aligned with.}
\label{alignment_table}
\end{table}
\end{proposition}
\begin{proof}
We show that if a learning algorithm is drug and target permutation equivariant, then it is uniformly badly aligned with $G^{\textnormal{\Autil-ODIT}}$. The other cases can be shown analogously.

Let $\bm{z}\in\mathcal{Z}^{\abs{\bm{z}}}$ be a sequence of training data, and let $\predictorvar_{\bm{z}}=\algo(\bm{z})$ be the random element inferred by $\algo$ from $\bm{z}$. Moreover, let $\drug\notin\mathcal{D}_{\bm{z}}$ and $\drug'\notin\mathcal{D}_{\bm{z}}$ be two OTS drugs. Then, obviously $\pi_{(\drug,\drug')}\cdot\bm{z}=\bm{z}$, where $\pi_{(\drug,\drug')}\in\operatorname{\Pi}_{\mathcal{D}}$ is a permutation that swaps the identities of drugs $\drug$ and $\drug'$.
Permutation equivariance of a learning algorithm $\algo$ to drug identities implies
\begin{align}\label{symmetryoftrsetimpliessymmetryofdist}
\begin{aligned}
\operatorname{P}[\predictorvar_{\bm{z}}\in\mathcal{F}]&=\operatorname{P}[\predictorvar_{\pi_{(\drug,\drug')}\cdot\bm{z}}\in\mathcal{F}]\\&=\operatorname{P}[\predictorvar_{\bm{z}}\in\pi_{(\drug,\drug')}\cdot\mathcal{F}]
\end{aligned}\;,
\end{align}
indicating that the distribution of $\predictorvar$ is symmetric with respect to OTS drugs. Let us fix a $2\times 2$ design of data:
\begin{align}\label{dataquadruplerepeated}
\begin{aligned}
&z &&=(\drug,\targ, y)  &&\phantom{WW}z^* &&=(\drug,\targ^*,y^*)\\%[1.5ex]
&z'&&=(\drug',\targ,y') &&\phantom{WW}z'^*&&=(\drug',\targ^*,y'^*)\end{aligned}
\end{align}
such that
\[
\left(\begin{array}{cc}
z&\phantom{W}z^*\\[1.0ex]
z'&\phantom{W}z'^*
\end{array}\right)
\in\mathcal{C}^{\textnormal{IC}}.
\]
\newcommand{\auxvar}{Q}%
Moreover, let us denote $\auxvar=\predictorvar_{\bm{z}}(\drug,\targ),\auxvar^*=\predictorvar_{\bm{z}}(\drug,\targ^*),\auxvar'=\predictorvar_{\bm{z}}(\drug',\targ),\auxvar'^*=\predictorvar_{\bm{z}}(\drug',\targ^*)$ for the values of $\predictorvar_{\bm{z}}$ on the $2\times 2$-design (\ref{dataquadruplerepeated}). For these variables, the symmetry (\ref{symmetryoftrsetimpliessymmetryofdist}) reduces to
\begin{align}\label{drugsymmetry}
\operatorname{P}\left[\left(\begin{array}{cc}
\auxvar&\phantom{W}\auxvar^*\\[1.0ex]
\auxvar'&\phantom{W}\auxvar'^*
\end{array}\right)\in\mathcal{Q}\right]
=\operatorname{P}\left[\left(\begin{array}{cc}
\auxvar'&\phantom{W}\auxvar'^*\\[1.0ex]
\auxvar&\phantom{W}\auxvar^*
\end{array}\right)\in\mathcal{Q}\right]
\end{align}
for any measurable subset $\mathcal{Q}\subseteq\mathbb{R}^4$.
Then, the expected interaction concordance on (\ref{dataquadruplerepeated}) is
\begin{align}
\operatorname{E}\left[\kernelf^{\textnormal{\Autil}}_\predictorvar(z,z^*,z',z'^*)\right]&=
\operatorname{E}[H\left(\auxvar-\auxvar^*-\auxvar'+\auxvar'^*\right)]\nonumber\\
&=\frac{1}{2}\operatorname{E}[H\left(\auxvar-\auxvar^*-\auxvar'+\auxvar'^*\right)]+\frac{1}{2}\operatorname{E}[H\left(\auxvar-\auxvar^*-\auxvar'+\auxvar'^*\right)]\nonumber\\
&=\frac{1}{2}\operatorname{E}[H\left(\auxvar-\auxvar^*-\auxvar'+\auxvar'^*\right)]+\frac{1}{2}\operatorname{E}[H\left(\auxvar'-\auxvar'^*-\auxvar+\auxvar^*\right)]\label{symmetryimplication}\\
&=\frac{1}{2}\operatorname{E}[H\left(\auxvar-\auxvar^*-\auxvar'+\auxvar'^*\right)]+\frac{1}{2}\operatorname{E}[1-H\left(\auxvar-\auxvar^*-\auxvar'+\auxvar'^*\right)]\label{heavisideimplication}\\
&=\frac{1}{2}\nonumber\;,
\end{align}
where equality (\ref{symmetryimplication}) follows from the symmetry (\ref{drugsymmetry}) and equality (\ref{heavisideimplication}) from the $H(-a)=1-H(a)$ property of the Heaviside function (\ref{heavisidefun}).
Accordingly, the expected interaction concordance is 0.5 for all ODIT data.

The consideration is analogous for IDOT and ODOT. With similar reasoning, we can show the expectation to be 0.5 for the drugwise concordance for IDOT, targetwise concordance for ODIT, and both for ODOT. In addition, for the ordinary (i.e., not drugwise nor targetwise) concordance, we can use the same approach to show that the expectation is 0.5 for ODOT data. For all other cases in Table \ref{alignment_table}, it is easy to find counter-examples of data distribution--learning algorithm combinations for which the expectation is larger than 0.5. These kinds of examples are shown in our simulation experiments.
\end{proof}

\section{Experiments}
An experimental study was conducted to demonstrate the behavior of the commonly used prediction performance estimators and validate the proposed \Aindex with different types of predictors and data. First, a simulation study was carried out to demonstrate the results of Proposition~\ref{alignmentproposition}, as described in Section~\ref{sec:simulation}. Then, experiments were carried out on benchmark drug-target data sets detailed in Section~\ref{sec:datasets}, with cross-validation approach described in Section~\ref{sec:crossvalidation}, using various machine learning algorithms described in Section~\ref{sec:methods}. The results of the experiments are presented in Section~\ref{sec:results}. The codes and links to all data used in the experiments are available at \url{https://github.com/TurkuML/IC-index-experiments}.

\subsection{Simulation}\label{sec:simulation}
To demonstrate the 20 learning problems covered in Table~\ref{alignment_table} in simple form,
we created simulated data representing a binary classification setting with $\{-1,1\}$ labels
resembling the classical XOR problem with imbalanced classes. The experiment was repeated 100~000 times and their results are averaged. An example of simulated data created during one of these repetitions is illustrated in Figure~\ref{fig:XOR_data}. The data has 200 ``drugs'' and 200 ``targets'', whose indices were used to calculate the class labels as follows:
\[
y=2\cdot(i_\drug > 20)\oplus (i_\targ \leq 40)-1\;,
\]
where $i_\drug\in\{1,\ldots,200\}$ and $i_\targ\in\{1,\ldots,200\}$ refer to the indices of drugs and targets, respectively, and $\oplus$ refers to the logical XOR function.
%The thresholds were selected so that 10~\% of the unique drug indices and 20~\% of the unique target indices were smaller than the thresholds. 
%In the simulation, 
Noise was added by randomly reversing class labels with 5~\% probability. Of these labels, 25~\% were randomly selected as known. 

The data contains the grand mean, drug main, target main and interaction effects in the following sense. There are more data labeled with 1 than with -1, indicating the presence of nonzero grand mean. Some drugs are associated to positive class labels more often than other drugs, implying a drug main effect. The same concerns the targets. Finally, since the XOR problem is inherently not additively separable, interaction effects as per Definition~\ref{2x2designeffects} naturally emerge.

%With the simulation, we considered all 20 learning problems covered in Table~\ref{alignment_table}. 
%During every repetition of the simulation, 50~\% of the drugs and targets were randomly selected as test drugs and test targets. A single test set then consists of data points whose both drug component is a test drug and target component is a test target and is used for all 20 learning problems. Separate training sets are then selected for the four different OTS cases so that $\mathcal{D}_{\bm{z}^{\textnormal{Test}}} \subseteq \mathcal{D}_{\bm{z}^{\textnormal{Train}}}$ and $\mathcal{T}_{\bm{z}^{\textnormal{Test}}} \subseteq \mathcal{T}_{\bm{z}^{\textnormal{Train}}}$ when the test drugs and targets are ITS drugs and targets, and $\mathcal{D}_{\bm{z}^{\textnormal{Test}}}$ and $\mathcal{T}_{\bm{z}^{\textnormal{Test}}}$ are disjoint from $\mathcal{D}_{\bm{z}^{\textnormal{Train}}}$ and $\mathcal{T}_{\bm{z}^{\textnormal{Train}}}$, respectively, when the test drugs and targets are OTS drugs and OTS targets.

\begin{figure}
    \centering
    \includegraphics{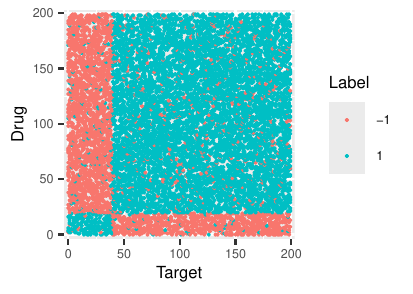}
    \caption{A visual example of imbalanced XOR data.}
    \label{fig:XOR_data}
\end{figure}

With the simulated data, we considered how five simple but representative permutation equivariant  learning algorithms do in terms of the 20 learning problems in Table~\ref{alignment_table}. These methods can also act as useful reference points when comparing learning algorithms' prediction performance for real-world problems. We refer to these learning algorithms as global sum (GS), drugwise sum (DS), targetwise sum (TS), sum of the drugwise and targetwise sums (SS), and product of the drugwise and targetwise sums (PS). The predictors they infer from a sequence of training data can be expressed simply as:
\begin{align*}
\predfun^{\textnormal{GS}}\left(d,t\right) &= \sum_{i=1}^{\arrowvert\bm{z}^{\textnormal{Train}}\arrowvert} y_{i}
\\
\predfun^{\textnormal{DS}}\left(d,t\right) &= \sum_{i=1}^{\arrowvert\bm{z}^{\textnormal{Train}}\arrowvert} y_{i}\cdot\delta[\drug=\drug_i]
\\
\predfun^{\textnormal{TS}}\left(d,t\right) &= \sum_{i=1}^{\arrowvert\bm{z}^{\textnormal{Train}}\arrowvert} y_{i}\cdot\delta[\targ=\targ_i]
\\
\predfun^{\textnormal{SS}}\left(d,t\right)&=\predfun^{\textnormal{DS}}\left(d,t\right)+\predfun^{\textnormal{TS}}\left(d,t\right)
\\
\predfun^{\textnormal{PS}}\left(d,t\right)&=\predfun^{\textnormal{DS}}\left(d,t\right)\cdot\predfun^{\textnormal{TS}}\left(d,t\right)
\end{align*}
where $\delta$ is the indicator function. Using the terminology given per Definition~\ref{constancy_and_linearity}, GS learns a constant function that always predicts the sum of the labels in the entire training data. DS learns a target symmetric function that predicts the sum of the labels in a subset of the training data containing only the triplets whose drug component was the same as for the test pair. TS learns analogous drug symmetric predictors. The DS and TS were then used as components of the other two methods. SS learns additively separable predictors and PS learns nonadditive ones.

As an additional point of reference, we evaluated second order polynomial regression (PR) based method with one-hot encoding of the drug and target indices, practically making it permutation equivariant. %as feature representation, second degree polynomials and one iteration of rank-one sub-problems by
PR is based on the algorithm described in Section~\ref{sec:PR}.
%This ensures that the polynomial regression model can express XOR function but has no side information usable for generalizing to IDOT, ODIT or ODOT data. 
Note also that, while the five simple learning algorithms are deterministic, the training algorithm for PR is a randomized one.

To estimate the learning algorithms' prediction performance for the 20 considered learning problems, we used in each of the 100000 repetitions roughly a quarter of the data to form the training data and another quarter to form the test data, such that they fulfill the equations (\ref{cvestimator}) and (\ref{samplesplit}). The performance estimates on test data were averaged over the repetitions. The results (see Figure \ref{fig:XOR_results}) are as expected by Proposition~\ref{alignmentproposition}. %Performances were measured by \Aindex, \Cindex (global, drugwise and targetwise) and accuracy. Average values and 95~\% credible intervals were calculated over the repetitions of the simulation. 
The 95~\% credible intervals show that, as expected, the prediction performances of the deterministic learning algorithms are always exactly 0.5 whenever they are badly aligned with the learning problem as per Proposition~\ref{alignmentproposition}. The mean prediction performance of PR is also roughly 0.5 for the problems it is badly aligned with, but the credible interval may stretch even up to 0.6 due to PRs randomized nature.

Better than random average \Aindex values are obtained only with PS and PR, because they are the only learning algorithms able to infer nonadditive predictors. Moreover, this takes place only on IDIT data, because no side information beyond the categorical drug and target identities is available. Better than random \CDindex can be obtained on IDIT and ODIT data by the methods able to model target main effects on them, namely TS, SS, PS and PR on the former and TS, SS and PR on the latter. The results are analogous for \CTindex. Better than random expected \Cindex can be obtained whenever either drug or target main effects can be modeled, which is for permutation equivariant algorithms impossible only for ODOT data. Finally, better than random binary classification accuracy can be obtained whenever the learning algorithm can model the grand mean. Overall, the results demonstrate that for all considered performance measures except \Aindex, fairly trivial methods can achieve good performance simply by modeling grand mean, drug main or target main effects. Further, the results of PR emphasize the well-known risk of getting falsely promising results with randomized learning algorithms on too small test data, even if their expected prediction performance would be 0.5.

\begin{figure}
    \centering
    \includegraphics{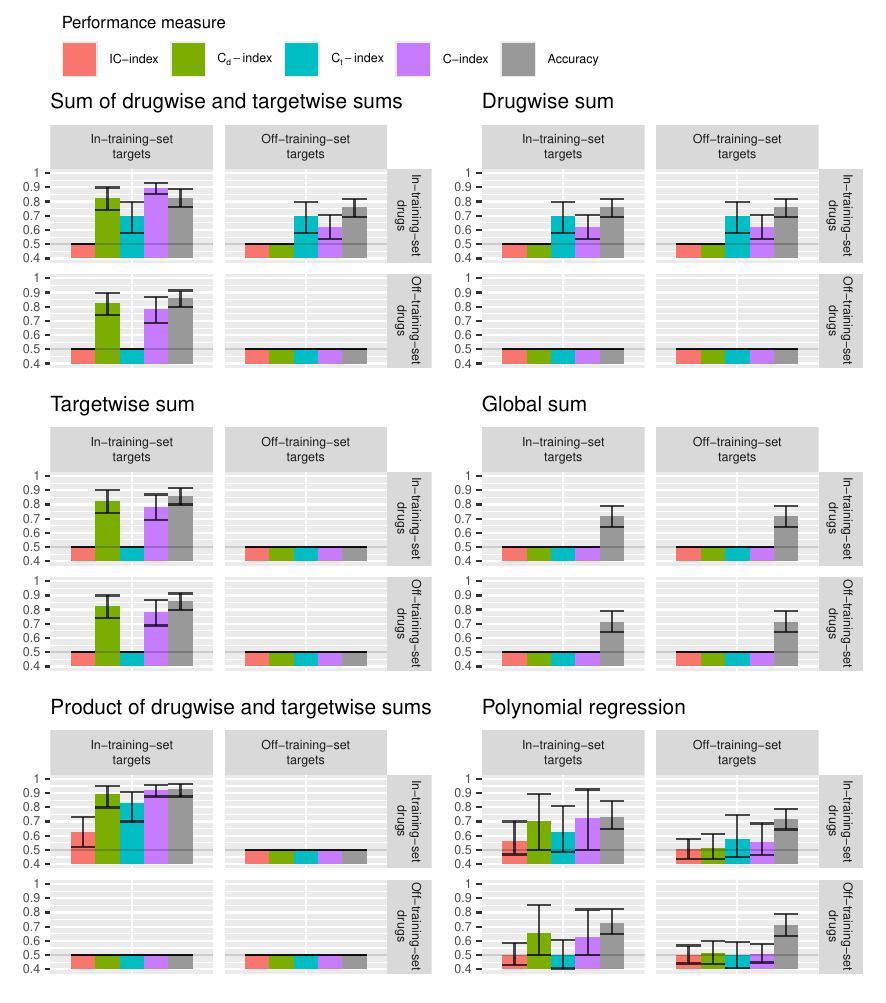}
    \caption{Average prediction performances on test data and symmetric 95 \% credible intervals over $10^6$ repetitions using imbalanced XOR data for the six considered learning algorithms.
    }
    \label{fig:XOR_results}
\end{figure}

\subsection{Data sets}\label{sec:datasets}

Real world DTA prediction experiments were run on seven drug-target data sets, whose characteristics, namely the numbers of drugs $\drsize$, targets $\tasize$, known DTA values $\trsize$, densities (den. \%)) and types of DTA values ( either continuous or binary), are presented in Table~\ref{tab:data_characteristics}. Each data consists of three matrices: feature matrices for drugs and targets as well as a matrix containing the DTA values. The feature matrices represent varying types of similarities between the elements within their respective domains, and thus are matrices of sizes ${\drsize \times \drsize}$ and ${\tasize \times \tasize}$. The DTA value matrix is of size $\drsize \times \tasize$.
For some of the data sets, all DTA values are known (i.e., density is 100\%), while others may only have a small subset of them available. Next, we provide a more detailed description of each data set. 

\begin{table}[h]
\centering
{\small
    \centering 
    \begin{tabular}{lrrrrl} \hline
    \hline \\ [-2.0ex]
        {\bf Data set name} & $\drsize$ & $\tasize$ & $\trsize$ & den. \% & {\bf DTA type}\\ \hline
        Davis & 68 & 442 & 30 056 & 100& Continuous \\
        Metz & 1 421 & 156 & 93 356 & 42& Continuous \\
        KIBA & 2 111 & 229 & 118 254 & 24& Continuous \\
        Merget & 2 967 & 226 & 167 995 &25& Continuous \\
        GPCR & 223 & 95 & 21 185 & 100 & Binary \\
        Ion Channels & 210 & 204 & 42 840 & 100& Binary \\
        Enzymes & 445 & 664 & 295 480 & 100& Binary \\\hline
\hline \\
    \end{tabular}
    \caption{Data set characteristics.}
\label{tab:data_characteristics}
    }
\end{table}

Data sets, that we call here shortly as Davis and Metz, are biochemical selectivity assays for clinically relevant kinase inhibitors by \citet{davis2011comprehensive} and \citet{metz2011navigating}, respectively. In these kinase disassociation constant (Davis) and kinase inhibition constant (Metz) data sets, the smaller the bioactivities, the higher the affinity between the chemical compound (drug) and the protein kinase (target). The drug feature matrices are based on the chemical properties of the drug compounds, and they contain structural fingerprint similarities between two drugs computed with 2D Tanimoto coefficients. The target feature matrices are based on genomic data, and they contain the normalized version of the Smith-Waterman scores \citep{SmithWaterman1981} between two targets. The DTA values represent dissociation constants in Davis data and inhibition constants in Metz data.

The kinase inhibitor bioactivity data set (KIBA) introduced by \citet{tang2014making}, integrates 
information captured by various bioactivity types, like IC50, kinase inhibition constant, and kinase disassociation constant, from multiple databases into a bioactivity matrix of 52 498 chemical compounds and 467 kinase targets, including 246 088 observations. Similarly to \citet{he2017simboost}, we only consider drugs and targets with more than ten DTA observations from the original data set, resulting in a data set of 2111 drugs and 229 targets with 24\% density. %The binding affinities in this matrix are used as labels while drug and target feature matrices are given similarly to Davis and Metz.

Data set, that we refer shortly as Merget, is a comprehensive kinome-wide drug–target binding affinity map originally generated by \citet{merget2017profiling}. 
In our experiment, we use it in the updated form described by \citet{cichonska2018learning}. Its DTA values are created by processing the affinity values from original Merget and updating them with the ChEMBL bioactivities by \citet{sorgenfrei2018kinome}. Since the original map is extremely sparse, it only involves drugs with at least 1\% of measured bioactivity values across the kinase panel, and also kinases with kinase domain and ATP binding pocket amino acid sub-sequences available in PROSITE \citep{Sigrist2013}, resulting in 2967 drugs, 226 protein kinases and 167 995 binding affinities. Feature matrices are selected from the sets of kernels computed by \citet{cichonska2018learning}. The feature matrix for drugs contains 1024-bit fingerprint based on the shortest paths between atoms, taking into account ring systems and charges, and the feature matrix for targets contains amino acid sub-sequences of ATP binding pockets and amino acid properties.

Finally, we applied the widely used binary DTA data sets GPCR, Ion Channels, and Enzymes, comprising compounds targeting pharmaceutically relevant proteins \citep{yamanishi2008prediction}. The DTA values within these datasets is obtained from KEGG BRITE, BRENDA, SuperTarget, and DrugBank databases, resulting in binary DTA matrices. Compound similarity scores, used as a feature matrix for drugs, are computed using the SIMCOMP score \citep{Hattori2003}, while protein sequence similarity scores are computed using the normalized Smith-Waterman score.

\subsection{Cross-validation}\label{sec:crossvalidation}
We applied the following variants of a 9-fold cross-validation approach for our experiments. First, both the drugs and targets were split into three groups. Then, the nine folds are formed from all combinations of the drug and target groups. In the cross-validation, each fold is used once as a test set. During each round, training data is formed from the other folds so that it fulfills the learning problem specific condition as per Definition~\ref{settingrestrictionsdef}. We also form a separate validation data for hyper-parameter selection so that it, together with training data, fulfills the same condition as the train-test split. We selected the hyper-parameter values using squared error as a surrogate for all considered base utilities, since using the utilities directly made no practical difference. The predictions for the test set are obtained with the predictor learned on the training set, which performed the best in the validation phase. We compute the performance measures for the test folds and average the fold-wise values.

\subsection{Methods}\label{sec:methods}
Making use of the proposed \Aindex, as well as the different variants of C-index, we investigate how different types of well-known machine learning algorithms can capture interaction effects in practical DTA prediction problems. Next, we give a more detailed description of the methods and their hyper-parameters.

\subsubsection{Pairwise kernel ridge regression}

As the first set of learning algorithms, we use the following variants of the pairwise kernel ridge regression method \citep{viljanen2021generalized}. Pairwise kernels can be considered as similarities between drug-target pairs. Here, we evaluate two pairwise kernels, namely an additive (linear pairwise kernel) and multiplicative (Kronecker product pairwise kernel) combination of two domain kernels. By domain kernel, we refer to a kernel function over drugs or over targets. We applied both linear and Gaussian RBF kernels as drug and target domain kernels, so that we have the following four variations. LR(linear), LR(Gaussian), KR(linear) and KR(Gaussian), where LR and KR refer to linear and Kronecker pairwise kernel ridge regression, respectively, and (linear) and (Gaussian) to the two domain kernels. As the description suggest, ridge regression with linear pairwise kernel (i.e. both LR(linear) and LR(Gaussian)) cannot model interaction effects, because the inferred predictors necessarily consists only of the drug main, target main and constant terms in the decomposition (\ref{eqn:decomposition}), even if the domain kernel is Gaussian. In contrast, both methods using the Kronecker product pairwise kernel can also capture the interaction effects.

We fit the pairwise kernel ridge regression models using the CGKronRLS solver from the RLScore software library (version 0.8.1.) \citep{pahikkala2016rlscore}.  With the Gaussian RBF kernels, the kernel width parameter $\mu=10^{-5}$ was used as recommended by \citet{airolapahikkala2018}. Maximum number of iterations was set to 1~000. Early stopping with a lag of 50 iterations was used to speed up the calculations and avoid overfitting. In other words, if the performance on validation data was not improved over the latest 50 iterations, the execution was terminated even though the maximum number of iterations was not reached yet. Regularization parameter was selected from $\left\lbrace 2^{-10}, 2^{-5}, 2^{-4}, 2^{-3}, 2^{-2}, 2^{-1}, 2^{0}, 2^{1}, 2^{2}, 2^{3}, 2^{4}, 2^{5}, 2^{10}\right\rbrace$.

\subsubsection{Polynomial regression} \label{sec:PR}
In polynomial regression the feature vector contains the linear regression part and all possible second and higher order polynomial terms of the drug and target features. The more higher order terms there are, the more complex interactions can be learned. We used latent tensor reconstruction based method \citep{szedmak2020solution} to learn the predictor, with the implementation from the Multiview (0.12.1) library. The method solves an optimization problem similar to kernel ridge regression. We chose to use second degree polynomials (parameter order = 2), indicating that there are two vectors of parameters whose values are optimized in each rank-one sub-problem. Parameter rank is related to regularization by denoting the number of rank-one problems to be solved, and was selected from a set $\left\lbrace 10, 20, 30, 40, 50, 60, 70, 80 \right\rbrace$.

\subsubsection{k-Nearest Neighbors}\label{sec:knn}

The $k$-Nearest Neighbors regressor \citep{FixEvelyn1951DA-N}, is a supervised learning algorithm used for predicting continuous outcomes. Its prediction for a new datum's outcome is the weighted average of the outcomes of the $k$ nearest training data, providing a simple yet effective method for regression tasks. For this algorithm the drug and target features were concatenated into one feature vector. 

We trained the $k$-nearest neighbors regressor using the scikit-learn \citep{Pedregosa2011scikit} implementation (version 1.0.2). We chose the default Euclidean distance, the neighbors were uniformly weighted, and we searched through a range of different values of neighbors from $\{5,10,30,50,75,100\}$. 

\subsubsection{Random Forest}\label{sec:rf}
Random forest is an ensemble learning method that constructs multiple decision trees during training and outputs the mode of the classes for classification or the mean prediction for regression. It tries to remedy the overfitting problem that single decision trees are prone to by combining the predictions from multiple trees. Random forest takes advantage of bootstrap aggregating (bagging) described by \citet{breiman1996bagging}, where multiple subsamples are drawn from the training set with replacement. Random forest may also take advantage of feature bagging described by \citet{HoTinKam1998Trsm}.

We used the scikit-learn implementation of random forest (version 1.0.2), which trains the individual trees using the CART method described by \citet{alma993704525405971}. We drew subsamples of the same size as the training set, used all of the features for each tree, and tried the total number of 100, 200, and 300 trees in our experiments.

\subsubsection{Extreme Gradient Boosting}\label{sec:xgboost}
Extreme gradient boosting \citep{ChenG16} belongs to the gradient boosting framework where multiple weak models are combined iteratively to form a stronger model. On each iteration, another weak model is fitted to the residual of the previous model in an attempt to correct its errors. In particular, XGBoost  \citep{ChenG16} is a gradient boosting library that implements decision tree boosting. We trained the XGBoostRegressor from the  xgboost (1.7.4) library with the default squared loss, trying 100, 125, and 150 learners.

\subsubsection{Feedforward Neural Networks}\label{sec:neuralnets}
A feedforward neural network is a fundamental architecture in machine learning, where information flows unidirectionally from input to output layers. We trained the standard feedforward network with dropout regularisation using the Adam optimizer to serve as a base model, making use of Keras and Tensorflow libraries. The optimal hyper-parameter combination for the network that achieved the best performance on the validation set was selected by performing a grid search over the following parameter ranges: the number of layers $\left\lbrace 2, 3\right\rbrace$, the dropout ratio from $\left\lbrace 0.05, 0.1, 0.2, 0.25\right\rbrace$, the number of epochs $\left\lbrace 1,..,200\right\rbrace$, the batch size from $\left\lbrace 64, 256\right\rbrace$, the learning rate from $\left\lbrace 0.005, 0.001\right\rbrace$, and the number of neurons in each layer for GPCR, Ion Channels and Enzymes, from the set of $\left\lbrace (1024,1024,512) ,(512,512,256) \right\rbrace$ and for Davis, Mets, KIBA and Merget from $\left\lbrace (2048,2048,1024,512) ,(1024,1024,512,256),(512,512,256,128) \right\rbrace$.

\subsubsection{DeepDTA and GraphDTA}

DeepDTA by \citet{Ozturk2018deepdta} and GraphDTA by \citet{Nguyen2020graphdta} are deep learning methods specially developed for predicting DTAs. Instead of utilizing drug-drug and target-target similarities, they use SMILES strings collected from Pubchem and protein sequences obtained from UniProt. These strings and sequences do not directly exist for all data sets in Table~\ref{tab:data_characteristics}. As a consequence, DeepDTA and GraphDTA are only applied to Davis and KIBA data sets.

Both DeepDTA and GraphDTA methods attempt to learn a hierarchical representation of proteins using 1D convolution filters. The methods differ in how they attempt to learn drug representations, DeepDTA uses 1D convolution filters for drugs as well, GraphDTA attempts to improve upon DeepDTA by converting the SMILES strings to graph representations using RDKit, an open-source cheminformatics library, and using various alternative graph convolution methods. For DeepDTA, the searched hyper-parameter range is the same as in the original paper \citep{Ozturk2018deepdta} for all settings. Nevertheless, better performance could be obtained with a different grid search, particularly when either or neither the drugs or/nor the targets are known. GraphDTA has no hyper-parameters to be selected other than the number of trained epochs over a range of 1 to 1000.

\begin{figure}[h]
    \centering
    \includegraphics{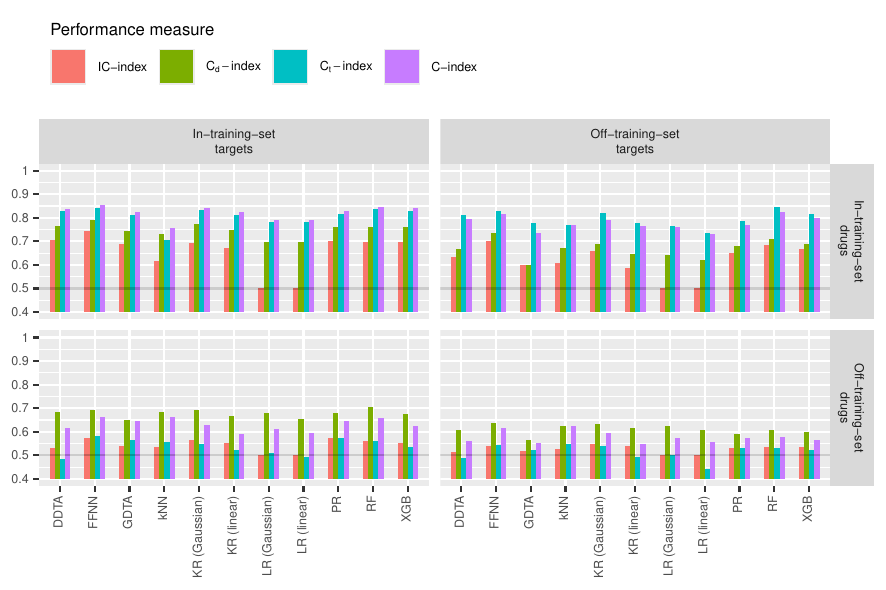}
    \caption{Results for the Davis data set.}
    \label{fig:results_Davis}
\end{figure}

\begin{figure}[h]
    \centering
    \includegraphics{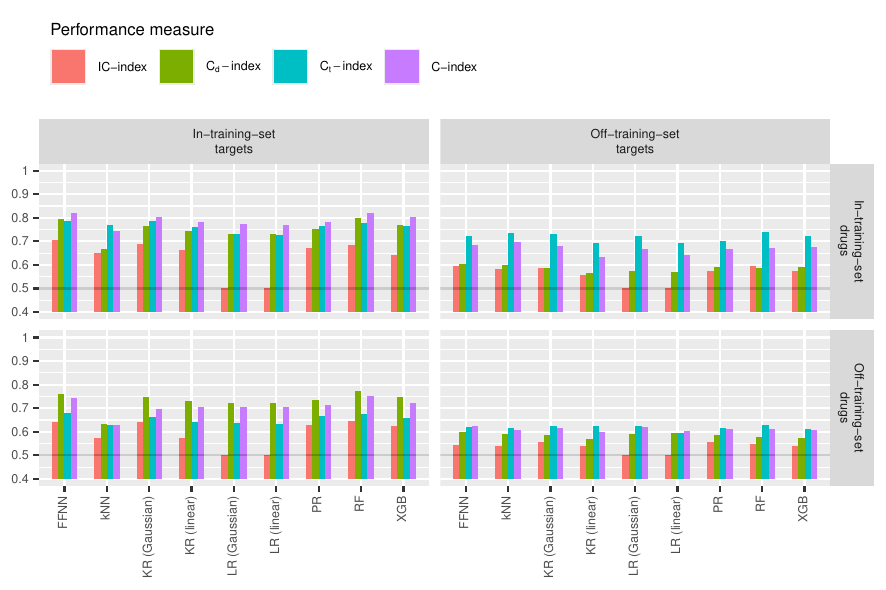}
    \caption{Results for the Metz data set.}
    \label{fig:results_Metz}
\end{figure}

\begin{figure}[h]
    \centering
    \includegraphics{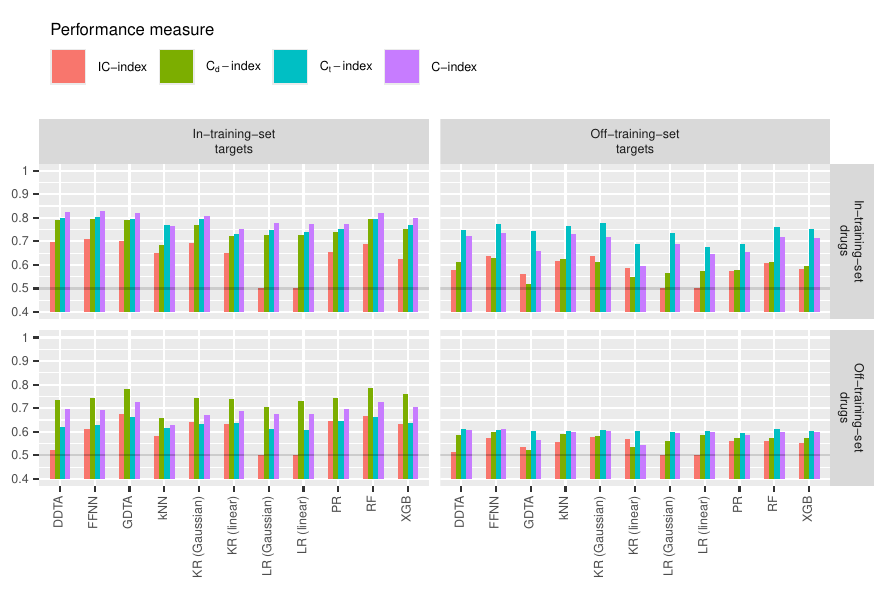}
    \caption{Results for the KIBA data set.}
    \label{fig:results_KiBA}
\end{figure}

\begin{figure}
    \centering
    \includegraphics{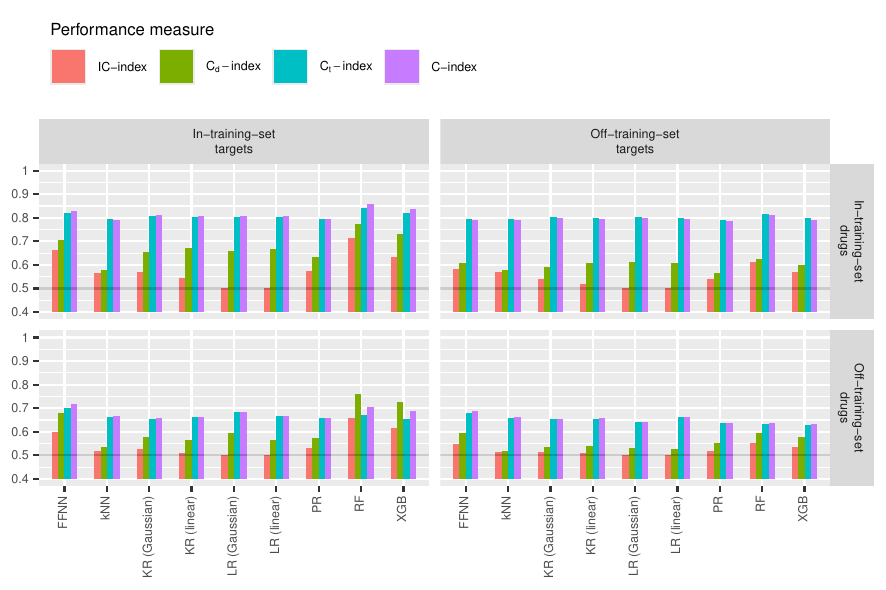}
    \caption{Results for the Merget data set.}
    \label{fig:results_Merget}
\end{figure}

\begin{figure}[h]
    \centering
    \includegraphics{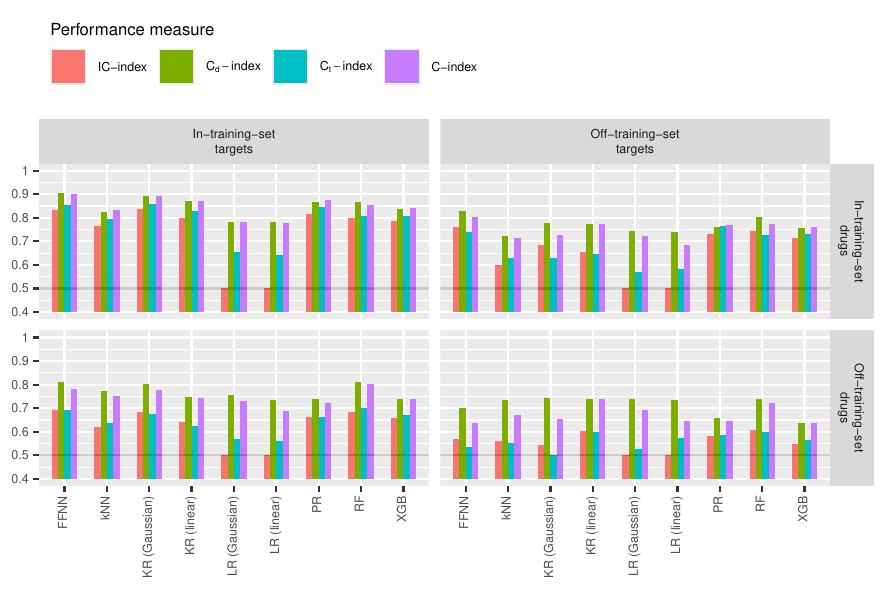}
    \caption{Results for the GPCR data set.}
    \label{fig:results_GPCR}
\end{figure}

\begin{figure}
    \centering
    \includegraphics{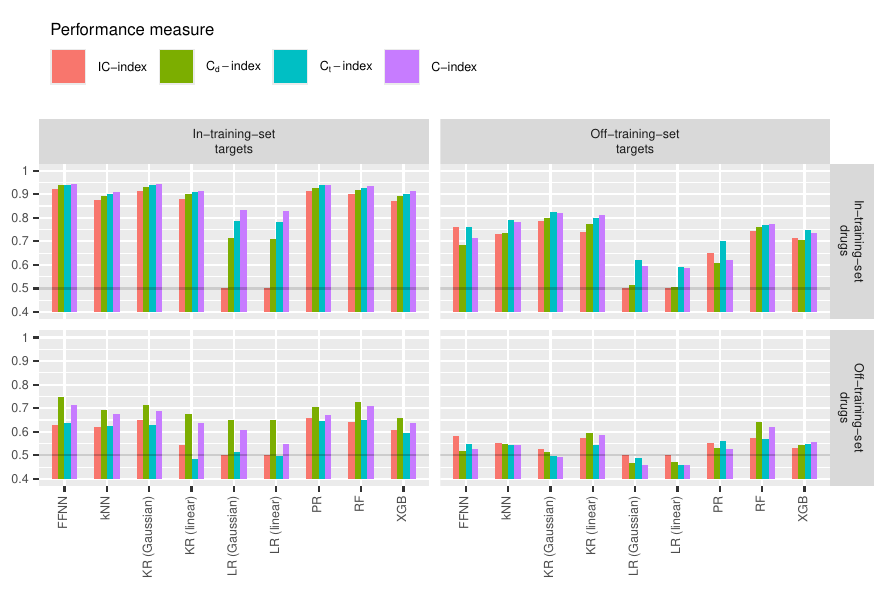}
    \caption{Results for the Ion Channels data set.}
    \label{fig:results_IC}
\end{figure}

\begin{figure}[h]
    \centering
    \includegraphics{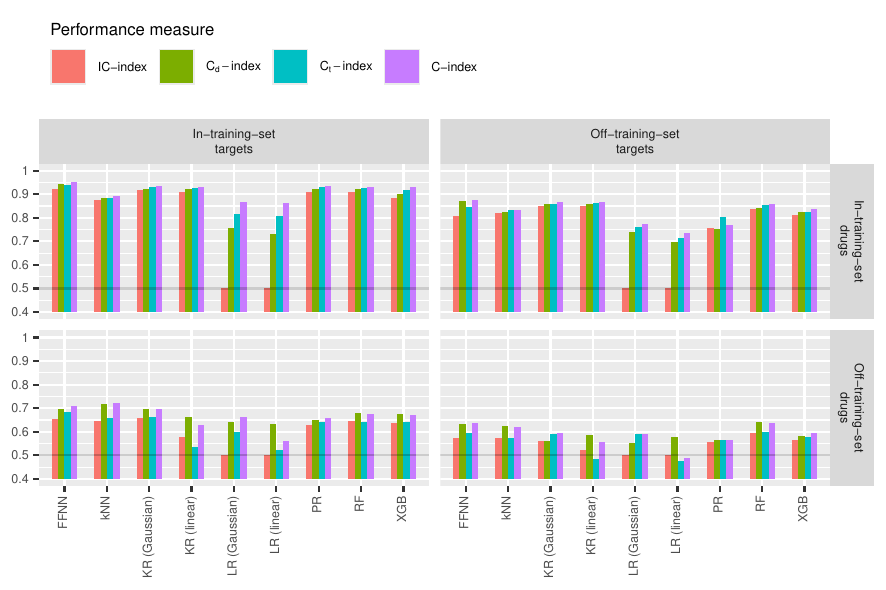}
    \caption{Results for the Enzymes data set.}
    \label{fig:results_E}
\end{figure}

\subsection{Results}\label{sec:results}

The prediction performances were measured by \Aindex, \CDindex, \CTindex and \Cindex. The results for the Davis, Metz, KIBA,  Merget, GPCR, Ion Channels and Enzymes data sets are presented in Figures~\ref{fig:results_Davis}--\ref{fig:results_E}. The methods are abbreviated as DeepDTA (DDTA), feedforward neural network (FFNN), GraphDTA (GDTA), k-nearest neighbors (KNN), pairwise Kronecker kernel ridge regression (KR), pairwise linear ridge regression (LR), polynomial regression (PR), random forest (RF) and XGBoost (XGB). For KR and LR the additional information in parenthesis is the type of domain kernels that were used.

First, we compare how different learning algorithms perform in terms of \Aindex. There is no method that always outperforms the others, but rather relative ranking of the methods depends on the data set and learning problem. Still, overall the basic FFNN is highly competitive, being in several experiments the top performing method (see especially Davis results in Figure~\ref{fig:results_Davis}), and almost always very close to the best methods. Notably, it also outperforms the more complex deep learning methods (DDTA and GDTA, see Figures~\ref{fig:results_Davis} and \ref{fig:results_KiBA}), with the exception of new target prediction on KIBA, where GDTA is the best method. RF is also a top performing method in several of the experiments, outperforming other methods on the Merget data (see Figure~\ref{fig:results_Merget}). KR (Gaussian) is often also among the top performing methods in the experiments (see e.g. Figure~\ref{fig:results_IC}), typically outperforming KR (Linear). The PR and XGB are also competitive, whereas kNN tends to have lower performance. LR (Gaussian) and LR (Linear) have always 0.5 \Aindex, as they are incapable of modeling the interaction effect.

In terms of the \Cindex based measures, the most notable difference is for the LR (Gaussian) and LR (Linear) methods. For the Ion Channels and Enzymes data sets (see Figures~\ref{fig:results_IC} and \ref{fig:results_E}) the methods inability to model the interaction effect also leads to much lower \Cindex compared to the other methods.
However, for the other data sets LR (Gaussian) and LR (Linear) tend to have only slightly lower performances than for the other methods. As an example, for Davis, Merget and Metz data sets, the LR methods outperform kNN on IDIT data in terms of C-index, even though the latter can capture the interaction effect and has higher than random \Aindex. For the other methods, the relative performance differences between them are roughly similar to those with \Aindex, with FFNN, RF, and often also KR (Gaussian) being the best performing methods.

Further, there are a number of trends that hold over all the data sets. First, prediction performances are the highest on IDIT data and the lowest on ODOT data. They are also often higher when generalizing to new targets than when generalizing to new drugs. Further, they are higher for data sets with binary valued outputs than for those with continuously valued ones, especially so on IDIT data. Secondly, \CTindex values are usually higher than \CDindex values when the drugs are known and the targets are new, and vice versa. This is consistent with the simulation results where no side information was used (see Figure~\ref{fig:XOR_results}), and hence \CDindex was on random level on IDOT data and likewise \CTindex on ODIT data.

\section{Discussion and conclusion}

In this work we introduced \Aindex, an estimator of interaction directions' prediction performance. For predictors, \Aindex calculates the fraction of correctly predicted interaction directions over all $2\times2$-designs on a given data. In other words, it is designed to assess whether a predictor truly goes beyond modeling the simple additive main effects of the two interacting objects under consideration and captures non-additive interactions between them. \Aindex is shown to be invariant to the grand mean and the main effects, and hence is genuinely a function of interaction effects only.

To shed some light on how capturing the interaction effects can be useful, say, in making allocation decisions of limited resources, we first consider more in detail the ``public health argument'' (see e.g. \citep{vanderweele2014tutorial}) and its variations mentioned in Section~\ref{sec:intro}. Let us interpret drug's affinity with target as the probability of curing a disease with the drug. Consider two different variants of a disease for which there exists an efficient but expensive drug with a limited number of available doses and a cheap but less efficient alternative whose doses are in abundance. Assume further that the probabilities of curing the two variants with the cheap drug are $1/10$ and $2/10$, and those of the expensive drug are $3/10$ and $5/10$. This forms a $2\times 2$-design of the two variants and two drugs that involves an interaction effect on the additive scale. The expected number of cures is maximized by allocating all available doses of the expensive drug for the second disease.

For learning algorithms, we presented variations of repeated hold-out based estimators of their expected \Aindex when both training and test data are randomly drawn. Similarly to the above described predictors, these estimators assess whether the learning algorithm under consideration truly captures the nonadditive interaction effects among new data not seen during training phase. In our experiments and in our prior work \citep{viljanen2021generalized}, we have shown that some learning algorithms, such as the ones inferring linear models, can achieve a high concordance index even without being able to model the interactions at all, resulting to the issues demonstrated by the above examples. Namely, if the predictor is additive with respect to the effects of the drugs and targets, resource allocation in the sense of the above example is not possible based on the predicted affinities. Moreover, as also discussed in Section~\ref{sec:intro}, such predictors always indicate the existence of a universal drug, the best cure for all diseases. This further underlines the need for the complementary \Aindex for which these predictions would have been completely invisible.

In particular, we presented distinct learning algorithms' prediction performance estimators indicating how well they generalize to drug-target pairs of which both the drug and target, only drug, only target or neither, have any known affinity values in the training data. %formed from in-training-set drugs and targets, in-training-set drugs and off-training-set targets, off-training-set drugs and in-training-set targets as well as off-training-set drugs and targets.
We show that if the learning algorithm under consideration is permutation equivariant with respect to drugs' and targets' categorical identities, then it can only learn to capture interaction effects for drug-target pairs, for which both drug and target have known affinities in the training data. %predict to learning algorithm's interaction predictions involving off-training set drugs or targets. Thus, the permutation equivariant learning algorithms can only learn to capture interactions for pairs of in-training-set object pairs, that is, object pairs whose both components have some known affinity values with some other object in the training data. 
Its practical consequence is the necessity of having side information on the objects beyond their distinct categorical identities, if the intent is to also capture interactions between new objects with no known affinity values in the training data. %requires side information on them.
%encoded into learning algorithm's inductive bias. 
Useful side information is often incorporated through feature representations of drugs and targets, such as their chemical structure.

We next mention a couple of imminent directions of further research. As noted by various authors (see e.g. \citet{vanderweele2014tutorial,Bours2021tutorial,Spake2023itdepends} and references therein), the concept of interaction is scale dependent. Within the additive scale considered in this paper, interaction is considered to take place if the affinities can not be expressed solely as sums of the main effects and grand mean. In contrast, interaction in multiplicative scale would indicate the quantity value's divergence from the products of the main effects and grand mean (see e.g. \citet{Ronkko2022eight}). Further, what is popularly referred to as the odds scale, is often used with binary valued quantities. Being able to capture interactions would enable the detection of some classical but on a first sight counterintuitive phenomena, such as the Simpson's paradox (see e.g. \citet{slavkovic2009algebraic,Norton2015SimpsonsParadox}). For example, if a drug appears to have a larger probability of curing two diseases than a second drug based on their aggregated success/failure counts, the second drug may have a larger success rate for both diseases when they are considered separately. In other words, the order gets reversed, when a confounding variable ``disease'' and its interaction with the ``drug'' variable is accounted. 

Since IC-index is defined as the fraction of correctly predicted directions of interaction on a sequence of data, another suite of prediction performance estimators could be defined for assessing how well the magnitudes of interactions (see e.g. \citet{vanderweele2019interaction}) are captured. By replacing the Heaviside function based utilities and prediction performance estimators considered in this paper with, say, squared error based ones, one obtains their regression error counterparts. While this type of loss functions have some history in the development of ranking algorithms \citep{werner2022review}, we are not aware of their use in interaction prediction performance estimation. 

\section*{Acknowledgments}

The authors would like to thank Oskari Heikkinen for his contributions to writing code for the experiments, and CSC – IT Center for Science, Finland, for generous computational resources. This work has been supported by Research Council of Finland (grants 340140, 340182, 345804, 345805 and 358868).

\bibliographystyle{unsrtnat} 
\bibliography{references}

%\newpage
%\section*{Supplementary materials}

\end{document}